\documentclass[11pt]{article}

\usepackage{amsmath , geometry , graphicx , amssymb, amsthm}
\usepackage{subcaption}
\usepackage{float}

\usepackage[style=authoryear,citestyle=authoryear,natbib=true,backend=bibtex]{biblatex}
\title{ Improved Robustness and Hyperparameter Selection in the Dense Associative Memory }
\author{ Hayden McAlister$^{1}$, Anthony Robins$^{1}$, Lech Szymanski$^{1}$ }
\date{
  $^{1}$School of Computing, University of Otago, Dunedin, New Zealand.
}

\begin{document}
\maketitle
\pagebreak

\begin{abstract}
    The Dense Associative Memory generalizes the Hopfield network by allowing for sharper interaction functions. This increases the capacity of the network as an autoassociative memory as nearby learned attractors will not interfere with one another. However, the implementation of the network relies on applying large exponents to the dot product of memory vectors and probe vectors. If the dimension of the data is large the calculation can be very large and result in imprecisions and overflow when using floating point numbers in a practical implementation. We describe the computational issues in detail, modify the original network description to mitigate the problem, and show the modification will not alter the networks' dynamics during update or training. We also show our modification greatly improves hyperparameter selection for the Dense Associative Memory, removing dependence on the interaction vertex and resulting in an optimal region of hyperparameters that does not significantly change with the interaction vertex as it does in the original network.
\end{abstract} 

\section{Introduction}

Autoassociative memories are a class of neural networks that learn to remember states, typically also allowing nearby states to iterate towards similar learned states. These networks act as memories for the learned states, reconstructing lost information and correcting errors in probe states. The Hopfield network \citep{Hopfield1982, Hopfield1984} is perhaps the most studied model in the class. However, as with all autoassociative memories, the Hopfield network suffers from capacity issues --- the number of states that can be stored in a network without error is limited. In the Hopfield network with Hebbian learning, this has been shown to be roughly \(0.14N\) for a network of dimension \(N\) \citep{McEliece1987, Hertz1991}. The Dense Associative Memory, also known as the modern Hopfield network, generalizes the classical Hopfield network by introducing an interaction function parameterized by an interaction vertex \citep{KrotovHopfield2016, KrotovHopfield2018}. This function controls the range of the influence for learned states, allowing control of the sizes of the attractors and increasing the network capacity. \citet{KrotovHopfield2016} also introduce several other generalizations which are parameterized by additional hyperparameters relating to learning, including the initial learning rate, learning rate decay, momentum, learning temperature and the exponent on the error term. Additional hyperparameters were introduced such as the form of the interaction function, the number of memory vectors, and more. In effect, the Dense Associative Memory is a potentially more powerful autoassociative memory, but at the cost of increased complexity and reliance on hyperparameter tuning.

We focus on the implementation details of the Dense Associative Memory. In particular, we show the exact form given by Krotov and Hopfield suffers from issues relating to computation and numerical stability. The current form calculates the dot product between two vectors of length \(N\) then immediately applies a potentially large exponentiation based on the interaction function. This can cause inaccuracies in the floating point numbers used for computation, or even completely overflow them. In Section \ref{Section:ModificationFormalization} we show a modification to the original form --- a normalization and shifting of scaling factors --- that prevents the computational problems, and prove that the modifications do not change the network behavior for a specific class of interaction functions: homogenous functions. Fortunately, the typical interaction functions --- the polynomial interaction function (Equation \ref{Eqn:PolynomialInteractionFunction}) and rectified polynomial interaction function (Equation \ref{Eqn:RectifiedPolynomialInteractionFunction}) --- are in this class. We show our modifications do not alter the properties of the autoassociative memory, such as the capacity, but do appear to have emergent effects on the network over the course of training. In Section \ref{Section:HyperparameterTuning} we provide experimental results that show our modified network has a stable region of optimal hyperparameters across a wide range of interaction vertices. This is in comparison to the original network which had the optimal hyperparameters shift dramatically as the interaction vertex changed even for the same dataset. We also show that the optimal region of hyperparameters is no longer heavily dependent on the size of the data vectors, meaning applying the Dense Associative Memory to a new task will not require massively retuning the hyperparameter selections.

\section{Literature Review}

Our proposed method of shifting the scaling factors within the interaction function does not appear to have been proposed previously, and other implementations of the Dense Associative Memory do not seem to have included it. However, many implementations of the Dense Associative Memory use the feed-forward equivalence set forth by Krotov and Hopfield \citep{KrotovHopfield2016}. This equivalence allows the Dense Associative Memory to be expressed with some approximations as a feed-forward densely connected neural network with a single hidden layer. This architecture is much easier to implement using traditional deep learning libraries. The feed-forward equivalent model implicitly implements our proposed changes by selecting values of the scaling factor that negate terms arising from a Taylor expansion. This may explain why the feed-forward version of the model is more stable than the auto-associative version.

Normalization is a typical operation in neural networks. In autoassociative memories specifically, we may apply a normalization term to provide a constant power throughout network calculations, which ensures calculations are proportional only to the magnitudes of the learned weights rather than the magnitude of the probe vector. Even more specifically, in the Hopfield network this is typically achieved by using binary valued vectors. It has been shown networks using these vectors have the same behavior as networks using graded (continuous value) neurons \citep{Hopfield1984}. Normalization may also be applied in the learning rule, such as in the Hebbian learning rule \citep{Hebb1949}. Normalization in learning may be used to simply scale the weights into something more interpretable, as in the Hebbian, or to achieve a different behavior during training. For example, batch normalization aims to improve training by normalizing the inputs to a layer across a batch --- allowing the network to focus only on the variations in training data rather than the potentially overwhelming average signal \citep{Ioffe2015}. Layer normalization is a technique used in training recurrent neural networks and removes the dependence on batch size \citep{Ba2016}. These normalizations techniques are more complex than what we suggest. Our modifications are not aiming to supersede these techniques in the Dense Associative Memory but simply improve network stability  and practicality on an implementation level. Moreover, our suggestions do not exclude the possibility of using these other normalization techniques.

Networks related to the Dense Associative Memory have employed some normalization techniques in a similar manner to our work. Perhaps most closely related is the continuous, attention-like Hopfield network \citep{Ramsauer2021} which has shown promising results in the realm of transformer architectures. Ramsauer et al. normalize the similarity scores as we do but work over a slightly different domain: spherical vectors rather than bipolar vectors. While the vector magnitude is still constant, the network has changed rather significantly from the one introduced by Krotov and Hopfield which may slightly change the arguments we make below, although likely not considerably. However, we note that no analysis of the network stability in relation to floating point accuracy is made, and the remainder of our modifications are not applied (e.g. shifting scaling factors inside the interaction function), which our work expands on considerably. Further works have performed a similar normalization, showing there is a trend of applying this technique in network implementation --- albeit without noting why it is useful for network stability \citep{Millidge2022, Liang2022, AlonsoKrichmar2024}. Literature on Dense Associative Memory applications and derivatives discuss normalization either in a separate context or only tangentially. Extensive work has been done on contrastive normalization (a biologically plausible explanation of network behavior) in the Dense Associative Memory and its relation to the restricted Boltzmann machine \citep{KrotovHopfield2021}. Other works employ some of the more advanced normalization techniques, including some we discuss above such as layer normalization by treating the Dense Associative Memory as a deep recurrent network \citep{Seidl2021}. Again, these works do not consider shifting the scaling factors within the interaction function.

\section{Formalization of the Hopfield Network and Dense Associative Memory}

The Hopfield network defines a weight matrix based on the Hebbian of the learned states \(\xi\), indexed by \(\mu\):
\begin{align}
    \label{Eqn:HebbianLearningRule}
    W_{ji} & = \sum_{\mu} \xi^{\mu}_{j} \xi^{\mu}_{i}
\end{align}

The update dynamics for a probe state \(\xi\) are defined by the sign of the energy function, with updates being applied asynchronously across neurons:
\begin{align}
    \label{Eqn:ClassicalHopfieldUpdate}
    \begin{split}
        \xi_{i}^{(t+1)} &= \text{sign} \left( W \xi^{(t)} \right)_i \\
        &= \text{sign} \left( \sum_j W_{ji} \xi_j^{(t)} \right),
    \end{split}
\end{align}
where \(\text{sign}\) is the sign function, or hardlimiting activation function:
\begin{align}
    \label{Eqn:HardLimitingActivation}
    \text{sign}(x) =
    \begin{cases}
        1  & \text{if } x\geq0, \\
        -1 & \text{if } x<0.
    \end{cases}
\end{align}

The Dense Associative Memory has significantly different learning rules and update dynamics compared to the Hopfield network, as well as major architectural changes, such as using a set of memory vectors \(\zeta\) instead of a weight matrix \(W\). The Dense Associative Memory also does away with a simple energy function and instead uses the sign of the difference of energies:
\begin{align}
    \label{Eqn:ModernHopfieldUpdateOriginal}
    \xi_i^{(t+1)} & = \text{sign} \left[ \sum_{\mu} \left( F_n ( \zeta_{i}^{\mu} + \sum_{j \neq i} \zeta_{j}^{\mu} \xi_j^{(t)} ) - F_n ( -\zeta_{i}^{\mu} + \sum_{j \neq i} \zeta_{j}^{\mu} \xi_j^{(t)} ) \right) \right]
\end{align}
Where \(F_n\) is the interaction function, parameterized by interaction vertex \(n\). The interaction vertex controls how steep the interaction function is. Typical interaction functions are the polynomial
\begin{align}
    \label{Eqn:PolynomialInteractionFunction}
    F_n\left(x\right) & = x^n,
\end{align}
rectified polynomial
\begin{align}
    \label{Eqn:RectifiedPolynomialInteractionFunction}
    F_n\left(x\right) & =
    \begin{cases}
        x^n & \text{if } x\geq0, \\
        0   & \text{if } x<0,
    \end{cases}
\end{align}
or leaky rectified polynomial
\begin{align}
    \label{Eqn:LeakyRectifiedPolynomialInteractionFunction}
    F_n\left(x, \epsilon\right) & =
    \begin{cases}
        x^n         & \text{if } x\geq0, \\
        -\epsilon x & \text{if } x<0.
    \end{cases}
\end{align}
The Hopfield network behavior is recovered when using the polynomial interaction function in Equation \ref{Eqn:PolynomialInteractionFunction} and \(n=2\) \citep{KrotovHopfield2016, Demircigil2017}. Increasing the interaction vertex allows memory vectors to affect only very similar probe vectors, decreasing interference with other memory vectors.

The Hopfield network requires only the energy calculation of the current state for updates (Equation \ref{Eqn:ClassicalHopfieldUpdate}), while the Dense Associative Memory requires the calculation of the energy for the current state when neuron \(i\) is clamped on (value \(1\)) and clamped off (value \(-1\)). This is more computationally expensive but allows for updating when the interaction vertex is larger than \(2\) and the usual arguments for update convergence in the Hopfield network fail \citep{Hopfield1982, Hopfield1985}.

Instead of a weight matrix, the Dense Associative Memory uses a set of memory vectors, clamped to have values between \(-1\) and \(1\), but not necessarily corresponding to the learned states. Instead, the learned states are used to update the memory vectors in a gradient descent. The loss function used in the gradient descent is based on the update rule in Equation \ref{Eqn:ModernHopfieldUpdateOriginal}:
\begin{equation}
    \begin{gathered}
        \label{Eqn:ModernHopfieldLearningOriginal}
        L = \sum_{a} \sum_{i} \left( \xi_{a, i} - C_{a, i} \right)^{2m} \\
        C_{a, i} = \tanh \left[ \beta \sum_{\mu} \left( F_n ( \zeta_{i}^{\mu} + \sum_{j \neq i} \zeta_{j}^{\mu} \xi_{a, j}^{(t)} ) - F_n ( -\zeta_{i}^{\mu} + \sum_{j \neq i} \zeta_{j}^{\mu} \xi_{a, j}^{(t)} ) \right) \right]
    \end{gathered}
\end{equation}
Where \(a\) indexes over the learned states, and \(i\) indexes over the neurons. The predicted value of neuron \(i\) in state \(\xi_a\), \(C_{a, i}\), is bounded between \(-1\) and \(1\) by \(\tanh\) . Taking \((\xi_{a, i} - C_{a, i})\) gives an error term we can differentiate to obtain a gradient to optimize with. The new parameters \(m\) and \(\beta\) control the learning process. The error exponent \(m\) emphasizes of larger errors,  which can help training networks with larger interaction vertices \citep{KrotovHopfield2016, KrotovHopfield2018}. The inverse temperature \(\beta\) scales the argument inside the \(\tanh\) function, allowing us to avoid the vanishing gradients of \(\tanh\) as the argument grows largely positive or largely negative. Krotov and Hopfield suggest \(\beta = \frac{1}{T^n}\).

Updating by Equation \ref{Eqn:ModernHopfieldUpdateOriginal} and learning by Equation \ref{Eqn:ModernHopfieldLearningOriginal} has proven successful when all hyperparameters are tuned carefully. However, we note some issues when implementing the network according to these rules in practice, particularly relating to floating point precision. Inspecting the order of calculations in Equation \ref{Eqn:ModernHopfieldUpdateOriginal} and \ref{Eqn:ModernHopfieldLearningOriginal}: first the ``similarity score'' between a learned state and a memory vector is calculated \(\left( \pm \zeta_{i}^{\mu} + \sum_{j \neq i} \zeta_{j}^{\mu} \xi_{j}^{(t)} \right)\), effectively the dot product between two binary vectors of length equal to the network dimension \(N\). Next, this similarity score is passed into the interaction function, which will typically have a polynomial-like region such as in Equation \ref{Eqn:PolynomialInteractionFunction}, \ref{Eqn:RectifiedPolynomialInteractionFunction}, or \ref{Eqn:LeakyRectifiedPolynomialInteractionFunction}. If the interaction vertex \(n\) is large the memory vectors become prototypes of the learned states \citep{KrotovHopfield2016}, hence the similarity scores will approach the bound for the dot product of two binary vectors, \(N\). We may have to calculate a truly massive number as an intermediate value. For example, \(N=10^{4}\) and \(n=30\) will result in an intermediate value of \(10^{120}\). Single precision floating point numbers (``floats'') have a maximum value of around \(10^{38}\), while double precision floating point numbers (``doubles'') have a maximum value of around \(10^{308}\). In our example, we are already incapable of even storing the intermediate calculation in a float, and it would not require increasing the network dimension or interaction vertex considerably to break a double either. Furthermore, the precision of these data types decreases as we approach the limits, potentially leading to numerical instabilities during training or updating. Even in the update rule (Equation \ref{Eqn:ModernHopfieldUpdateOriginal}) where only the sign of the result is relevant, a floating point overflow renders the calculation unusable.

We propose a slight modification to the implementation of the Dense Associative Memory. Normalizing the similarity score by the network dimension \(N\) bounds the magnitude of the result to \(1\) rather than \(N\). Additionally, we propose pulling the scaling factor \(\beta\) inside the interaction function, so we can appropriately scale the value before any imprecision is introduced by a large exponentiation as well as controlling the gradient, making the network more robust. We show these modifications are equivalent to the original Dense Associative Memory specification in Section \ref{Section:ModificationFormalization}. In Section \ref{Section:HyperparameterTuning} we also show by experimentation that these modifications make the network temperature independent of the interaction vertex. This makes working with the Dense Associative Memory more practical, as it avoids large hyperparameter searches when slightly altering the interaction vertex.

\section{Modification and Consistency with Original}
\label{Section:ModificationFormalization}

Our modifications attempt to rectify the floating point issues by scaling the similarity scores \textit{before} applying the exponentiation of the interaction function. To justify our modifications we must show that the scaling has no effect on the properties of the Dense Associative Memory in both learning and updating. For the update rule, we will show the sign of the argument to the hardlimiting function in Equation \ref{Eqn:ModernHopfieldUpdateOriginal} is not affected as we introduce a scaling factor and move it within the interaction function. For learning, we will make a similar argument using Equation \ref{Eqn:ModernHopfieldLearningOriginal}.

\subsection{Homogeneity of the Interaction Function}

In parts of our proof on the modification of the Dense Associative Memory we require the interaction function to have a particular form. We require the sign of the difference of two functions remain constant even when a scaling factor is applied inside those functions; \(f(x)-f(y) = f(\alpha x) - f(\alpha y) \quad \forall a>0\). A stronger property (that is much easier to prove) is that of homogeneity:
\begin{align}
    f(\alpha x) = \alpha^k f(x) \qquad \forall \alpha > 0
\end{align}
with the exponent \(k\) known as the degree of homogeneity. Interaction functions that are not homogenous may still satisfy our modifications, but we find the proof easier with this stronger property.

\newtheorem{HomogeneityOfInteractionFunctions}{Theorem}[subsection]
\newtheorem{HomogeneityOfPolynomialInteractionFunction}[HomogeneityOfInteractionFunctions]{Lemma}
\begin{HomogeneityOfPolynomialInteractionFunction}
    \label{Lemma:HomogeneityOfPolynomialInteractionFunction}
    The polynomial interaction function (Equation \ref{Eqn:PolynomialInteractionFunction}) is homogenous.
\end{HomogeneityOfPolynomialInteractionFunction}
\begin{proof}
    \begin{align}
        \begin{split}
            F_n\left(\alpha x\right) & = \left( \alpha x \right)^n \\
            & =  \alpha^n  x^n \\
            & =  \alpha^n F_n(x) 
        \end{split}
    \end{align}
    
    Hence, the polynomial interaction function is homogenous, with degree of homogeneity equal to the interaction vertex \(n\).
\end{proof}

\newtheorem{HomogeneityOfRectifiedPolynomialInteractionFunction}[HomogeneityOfInteractionFunctions]{Lemma}
\begin{HomogeneityOfRectifiedPolynomialInteractionFunction}
    \label{Lemma:HomogeneityOfRectifiedPolynomialInteractionFunction}
    The rectified polynomial interaction function (Equation \ref{Eqn:RectifiedPolynomialInteractionFunction}) is homogenous.
\end{HomogeneityOfRectifiedPolynomialInteractionFunction}
\begin{proof}
    \begin{align}
        \label{Eqn:HomogeneityOfRectifiedPolynomialInteractionFunction}
        \begin{split}
            F_n\left(\alpha x\right) & = \begin{cases}
                \left( \alpha x \right)^n & \text{if } \left(\alpha x\right)\geq0, \\
                0   & \text{if } \left(\alpha x\right)<0,
            \end{cases} \\
            & = \begin{cases}
                \alpha^n  x^n & \text{if } \left(\alpha x\right)\geq0, \\
                0   & \text{if } \left(\alpha x\right)<0,
            \end{cases} \\
            & = \alpha^n \begin{cases}
                x^n & \text{if } x\geq0, \\
                0   & \text{if } x<0,
            \end{cases} \\
            & =  \alpha^n F_n(x) 
        \end{split}
    \end{align}
    Note that the sign of \(x\) is unchanged by scaling by \(\alpha > 0\), so we can change the conditions on the limits as we did in Equation \ref{Eqn:HomogeneityOfRectifiedPolynomialInteractionFunction}. Hence, the rectified polynomial interaction function is homogenous, with degree of homogeneity equal to the interaction vertex \(n\).
\end{proof}

\subsubsection{On Common Nonhomogenous Interaction Functions}

The leaky rectified polynomial interaction function (Equation \ref{Eqn:LeakyRectifiedPolynomialInteractionFunction}) is common in literature, alongside Equation \ref{Eqn:PolynomialInteractionFunction} and \ref{Eqn:RectifiedPolynomialInteractionFunction}. However, the leaky rectified polynomial is not homogenous. Empirically, we find it still behaves well under our modifications.

The Dense Associative Memory has been generalized further using an exponential interaction function \citep{Demircigil2017}. Another modification of the exponential interaction function has been used to allow for continuous states and an exponential capacity \citep{Ramsauer2021}. This interaction function has been analyzed in depth and linked to the attention mechanism in transformer architectures \citep{Vaswani2017}. For completion, we discuss our proposed modifications to the new, wildly popular interaction function:
\begin{align}
    \label{Eqn:ExponetialInteractionFunction}
    F\left(x\right) & = e^{x},
\end{align}
Clearly, the exponential interaction function is not homogenous, as 
\begin{align*}
    F(\alpha x) &= e^{\alpha x} \\
        &= e^{\alpha} e^{x} \\
        &= e^{\alpha} F(x)
\end{align*}
Since the constant \(\alpha\) does not have the form \(\alpha^k\) when pulled out of the function, the exponential function is not homogenous. However, we can analyze the exponential function specifically and relax the homogeneity constraint to show our modifications will not affect networks with exponential interaction functions. In particular, we need only show the sign of the difference between two exponentials is unaffected:
\begin{align*}
    \text{sign} \left[\alpha \left(F(x) - F(y)\right)\right] &= \text{sign} \left[\alpha \left(e^{x} - e^{y}\right)\right] \\
        &= \text{sign} \left[\alpha e^{x} - \alpha e^{y}\right] \\
        &= \text{sign} \left[e^{\log(\alpha)}e^{x} - e^{\log(\alpha)}e^{y}\right] \\
        &= \text{sign} \left[e^{\log(\alpha) x} - e^{\log(\alpha) y}\right].
\end{align*}
The behavior of the Dense Associative Memory using the exponential interaction function would be unchanged by normalizing the similarity scores before taking the exponential. This may help stabilize the continuous Dense Associative Memory and improve integrations in deep learning architectures. 

\subsection{Update in the Dense Associative Memory}
\label{Section:UpdateRuleModification}

We start with the right-hand side of Equation \ref{Eqn:ModernHopfieldUpdateOriginal}, introducing an arbitrary constant \(\alpha > 0\). We will then show this has no effect on the sign of the result, and we are free to choose \(\alpha = \frac{1}{N}\) to normalize the similarity scores by the network dimension.

\newtheorem{ModernHopfieldModificationTheorem}{Theorem}[subsection]
\newtheorem{UpdateRuleTheorem}[ModernHopfieldModificationTheorem]{Theorem}
\begin{UpdateRuleTheorem}
    \label{Theorem:UpdateRuleTheorem}
    The Dense Associative Memory, equipped with a homogenous interaction function, has unchanged update dynamics (Equation \ref{Eqn:ModernHopfieldUpdateOriginal}) when applying a scaling factor \(\alpha\) to similarity calculations inside the interaction function. That is:

    \begin{align*}
        & \textup{sign} \left[ \sum_{\mu} \left( F_n ( \zeta_{i}^{\mu} + \sum_{j \neq i} \zeta_{j}^{\mu} \xi_j^{(t)} ) - F_n ( -\zeta_{i}^{\mu} + \sum_{j \neq i} \zeta_{j}^{\mu} \xi_j^{(t)} ) \right) \right] \\
        = \;& \textup{sign} \left[ \sum_{\mu} \left( F_n ( \alpha ( \zeta_{i}^{\mu} + \sum_{j \neq i} \zeta_{j}^{\mu} \xi_j^{(t)} ) ) - F_n ( \alpha ( -\zeta_{i}^{\mu} + \sum_{j \neq i} \zeta_{j}^{\mu} \xi_j^{(t)} ) ) \right) \right].
    \end{align*}
\end{UpdateRuleTheorem}

\begin{proof}

The sign of any real number is unaffected by scaling factor \(\alpha > 0\):

\begin{align*}
    \begin{split}
        & \text{sign} \left[ \sum_{\mu} \left( F_n ( \zeta_{i}^{\mu} + \sum_{j \neq i} \zeta_{j}^{\mu} \xi_j^{(t)} ) - F_n ( -\zeta_{i}^{\mu} + \sum_{j \neq i} \zeta_{j}^{\mu} \xi_j^{(t)} ) \right) \right] \\
    = \;& \text{sign} \left[ \alpha \sum_{\mu} \left( F_n ( \zeta_{i}^{\mu} + \sum_{j \neq i} \zeta_{j}^{\mu} \xi_j^{(t)} ) - F_n ( -\zeta_{i}^{\mu} + \sum_{j \neq i} \zeta_{j}^{\mu} \xi_j^{(t)} ) \right) \right] \\
    = \;& \text{sign} \left[ \sum_{\mu} \left( \alpha F_n ( \zeta_{i}^{\mu} + \sum_{j \neq i} \zeta_{j}^{\mu} \xi_j^{(t)} ) - \alpha F_n ( -\zeta_{i}^{\mu} + \sum_{j \neq i} \zeta_{j}^{\mu} \xi_j^{(t)} ) \right) \right]
    \end{split}
\end{align*}
Moving \(\alpha\) within the interaction function \(F_n\) requires constraints on the interaction function. We require that the sign of the difference remains the same. A homogenous interaction function gives us this condition, although it is slightly stronger than is required. Using the assertion that \(F_n\) is homogenous:

\begin{align*}
    = \;& \text{sign} \left[ \sum_{\mu} \left( \alpha F_n ( \zeta_{i}^{\mu} + \sum_{j \neq i} \zeta_{j}^{\mu} \xi_j^{(t)} ) - \alpha F_n ( -\zeta_{i}^{\mu} + \sum_{j \neq i} \zeta_{j}^{\mu} \xi_j^{(t)} ) \right) \right] \\
    = \;& \text{sign} \left[ \sum_{\mu} \left( F_n ( \alpha^\prime ( \zeta_{i}^{\mu} + \sum_{j \neq i} \zeta_{j}^{\mu} \xi_j^{(t)} ) ) - F_n ( \alpha^\prime ( -\zeta_{i}^{\mu} + \sum_{j \neq i} \zeta_{j}^{\mu} \xi_j^{(t)} ) ) \right) \right] 
\end{align*}

Since the scaled factor \(\alpha^\prime\) is still arbitrary, we are free to select any (positive) value we like without changing the result.

\end{proof}

Therefore, our modified network's update rule will give the same behavior as the original update rule in Equation \ref{Eqn:ModernHopfieldUpdateOriginal}. Our modified update rule is given by:
\begin{align}
    \label{Eqn:ModernHopfieldUpdateModified}
    \xi_i^{(t+1)} &= \text{sign} \left[ \sum_{\mu} \left( F_n ( \alpha ( \zeta_{i}^{\mu} + \sum_{j \neq i} \zeta_{j}^{\mu} \xi_j^{(t)} ) ) - F_n ( \alpha ( -\zeta_{i}^{\mu} + \sum_{j \neq i} \zeta_{j}^{\mu} \xi_j^{(t)} ) ) \right) \right] 
\end{align}
All that is left is to choose a value for the scaling factor \(\alpha\). As discussed, we suggest choosing \(\alpha = \frac{1}{N}\), the inverse of the network dimension, such that the similarity scores are normalized between \(-1\) and \(1\), which nicely avoids floating point overflow. It may appear we are trading one floating point inaccuracy for another, as now our worst case would have small similarity scores (intermediate values close to \(0\)) mapped even closer to \(0\) by the polynomial interaction functions, where again floating point numbers are inaccurate. However, the failure case here is to set the value to exactly \(0.0\) rather than ``infinity'' or ``NaN'', and hence computation may continue albeit with reduced accuracy. Furthermore, once training has progressed slightly the memory vectors will likely be quite similar to the data, avoiding this problem. Finally, we could further tune the scaling factor if numerical instability is still a concern, as we have shown a general scaling factor is admissible. In practice, we found this was not required.

\subsection{Learning in the Dense Associative Memory}

Reasoning about the learning rule (Equation \ref{Eqn:ModernHopfieldLearningOriginal}) is slightly trickier than the update rule. We must ensure the network learning remains consistent rather than just the sign of the energy difference. Furthermore, there is already a scaling factor \(\beta\) present. We will show that we can pull the existing scaling factor within the interaction function and keep its intended action of shifting the argument of the \(\tanh\) function, and hence that we can achieve the same calculation as the original network with our modifications. The argument here is largely the same as in Section \ref{Section:UpdateRuleModification}.

\newtheorem{LearningRuleTheorem}[ModernHopfieldModificationTheorem]{Theorem}
\begin{LearningRuleTheorem}
    \label{Theorem:LearningRuleTheorem}
    The Dense Associative Memory, equipped with a homogenous interaction function, has unchanged learning behavior (Equation \ref{Eqn:ModernHopfieldLearningOriginal}) when moving the scaling factor \(\beta\) inside the interaction function evaluations, up to adjusting the scaling factor. That is:

    \begin{align*}
        &\tanh \left[ \beta^\prime \sum_{\mu} \left( F_n ( \zeta_{i}^{\mu} + \sum_{j \neq i} \zeta_{j}^{\mu} \xi_{j}^{(t)} ) - F_n ( -\zeta_{i}^{\mu} + \sum_{j \neq i} \zeta_{j}^{\mu} \xi_{j}^{(t)} ) \right) \right]\\
        = &\tanh \left[ \sum_{\mu} \left( F_n ( \beta (\zeta_{i}^{\mu} + \sum_{j \neq i} \zeta_{j}^{\mu} \xi_{j}^{(t)}) ) - F_n ( \beta (-\zeta_{i}^{\mu} + \sum_{j \neq i} \zeta_{j}^{\mu} \xi_{j}^{(t)}) ) \right) \right].
    \end{align*}
\end{LearningRuleTheorem}

\begin{proof}
    Equation \ref{Eqn:ModernHopfieldLearningOriginal} defines a loss function over which a gradient descent is applied to update the memory vectors \(\zeta\). To show this gradient descent is unchanged by moving the scaling factor \(\beta\) inside the interaction function evaluations, we focus on the predicted neuron value in Equation \ref{Eqn:ModernHopfieldLearningOriginal} and apply the same algebra as in Theorem \ref{Theorem:UpdateRuleTheorem} to take the scaling factor \(\beta^\prime\) inside the interaction function. Note that this also requires the homogeneity of the interaction function, and may alter the value the scaling factor \(\beta\), but will ensure the argument to the \(\tanh\) (and hence the gradient) remains the same. The exact gradient expression is eschewed here but remains unchanged from the original.    
\end{proof}
Therefore, our modified learning rule has the form:

\begin{equation}
    \begin{gathered}
    \label{Eqn:ModernHopfieldLearningModified}
    \mathcal{L} = \sum_{a} \sum_{i} \left( \xi_{a, i} - C_{a, i} \right)^{2m} \\
    C_{a, i} = \tanh \left[ \sum_{\mu} \left( F_n ( \beta (\zeta_{i}^{\mu} + \sum_{j \neq i} \zeta_{j}^{\mu} \xi_{j}^{(t)}) ) - F_n ( \beta (-\zeta_{i}^{\mu} + \sum_{j \neq i} \zeta_{j}^{\mu} \xi_{j}^{(t)}) ) \right) \right].
    \end{gathered}
\end{equation}

Krotov and Hopfield suggest a value of \(\beta = \frac{1}{T^n}\). We suggest a modified value of \(\beta = \frac{1}{NT}\). Since interaction functions of interest have a degree of homogeneity equal to the interaction vertex \(n\), shifting the scaling factor inside is effectively equivalent to taking the exponent of \(\frac{1}{n}\), so we can remove the exponent from Krotov and Hopfield's suggestion. Furthermore, as in Section \ref{Section:UpdateRuleModification} we suggest normalizing the similarity score by the network dimension to be bounded between \(-1\) and \(1\). It may seem alarming that we suggest massively lowering the similarity score in this equation, as it may affect the argument passed to the \(\tanh\) function and hence the magnitude of the gradients used in learning, but we can always simply rescale \(\beta\) using the temperature to increase this value again if required. However, the default behavior of the network now results in rapidly shrinking intermediate values during training, rather than exploding values that are often unrecoverable. By tuning \(\beta\) we can shift the argument to \(\tanh\) just as we could in the unmodified network while still avoiding floating point overflow.

\section{Hyperparameter Tuning}
\label{Section:HyperparameterTuning}

The original Dense Associative Memory suffered from very strict hyperparameter requirements. Furthermore, changing the value of the interaction vertex would significantly change the hyperparameters that would train the model well. We find that our modifications --- particularly, normalizing the similarity scores in the learning rule (Equation \ref{Eqn:ModernHopfieldLearningModified}) --- removed the dependence on the interaction vertex, meaning we can reuse the same hyperparameters for a task even as we change the interaction vertex.

We focus on the most important hyperparameters for learning: the initial learning rate and temperature. Other hyperparameters were tuned but did not display behavior as dramatic as we present here. We use a learning rate decay of \(0.999\) per epoch, a momentum of \(0.0\), and an error exponent of \(m=1.0\). We found similar results using a decay rate of \(1.0\) and higher values for momentum. We also found we did not require changing the error exponent \(m\), which \citet{KrotovHopfield2016} found to be useful in learning higher interaction vertices. This perhaps indicates we can remove this hyperparameter and simplify the network.

The network is trained on \(20\) randomly generated bipolar vectors of dimension \(100\). Even for the lowest interaction vertex \(n=2\) this task is perfectly learnable. Larger dimensions and other dataset sizes were tested with similar results. After training, we probe the network with the learned states; if the probes move only a small distance from the learned states, the network operates as an acceptable associative memory. We measure the average distance from the final, stable states to the learned states, for which a lower value is better. We repeat the experiment five times for each combination of hyperparameters. Select interaction vertices are shown below, and the full results can be found in Appendix \ref{Section:AppendixResults}. In particular, Appendix \ref{Appendix:ModifiedDim100LargeN} shows our results for interaction vertices up to \(n=100\); far above any interaction vertices documented in other literature.

\subsection{Original Network Hyperparameter Results}
\label{Section:OriginalNetworkAutoassociativeTask}

\begin{figure}[H]
    \centering
    \begin{subfigure}[t]{0.48\textwidth}
        \includegraphics[width=\textwidth]{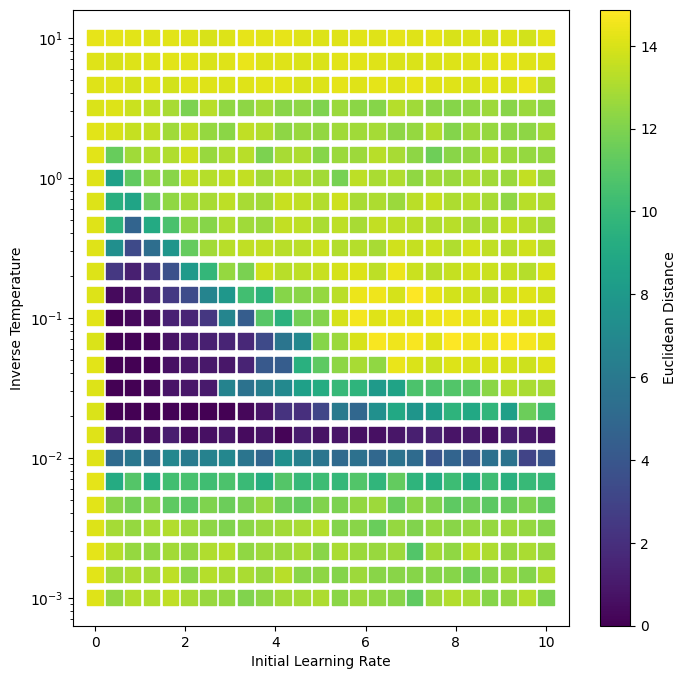}
        \caption{\(n=2\)}
    \end{subfigure}
    \begin{subfigure}[t]{0.48\textwidth}
        \includegraphics[width=\textwidth]{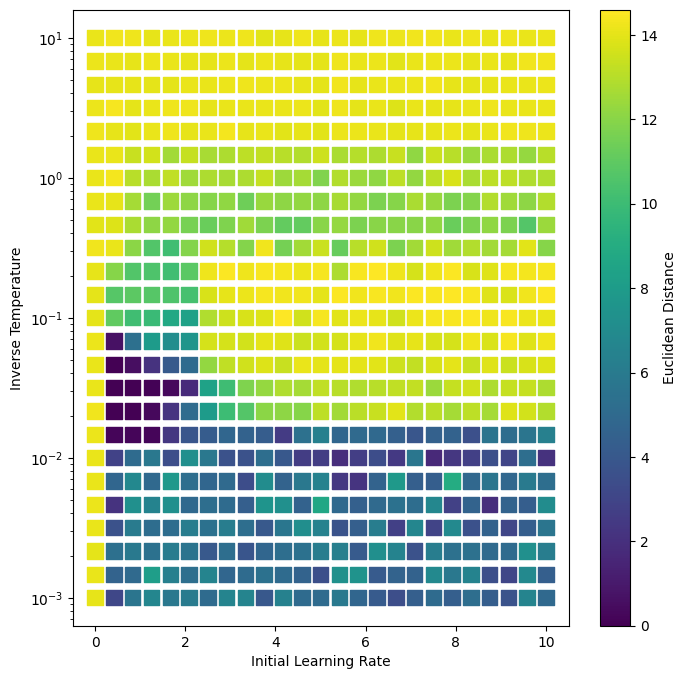}
        \caption{\(n=5\)}
    \end{subfigure}
    \hfill
    \begin{subfigure}[t]{0.48\textwidth}
        \includegraphics[width=\textwidth]{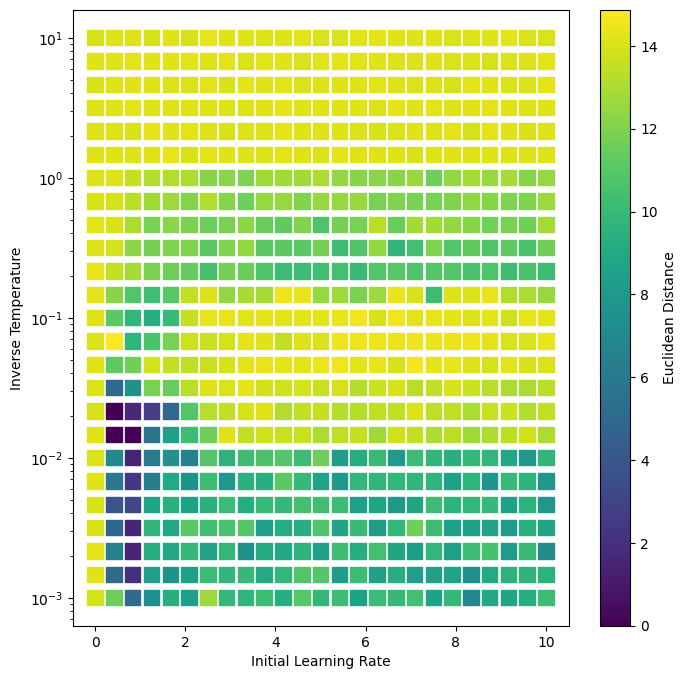}
        \caption{\(n=10\)}
    \end{subfigure}
    \begin{subfigure}[t]{0.48\textwidth}
        \includegraphics[width=\textwidth]{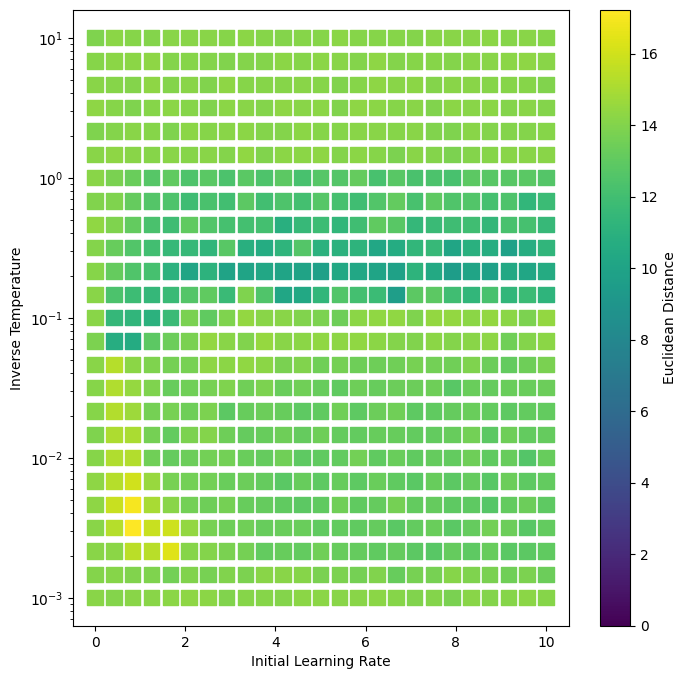}
        \caption{\(n=20\)}
        \label{Fig:OriginalModelHyperparameterSearchN20}
    \end{subfigure}
    
    \caption{Coarse hyperparameter search space for the original network, measuring the Euclidean distance between learned states and relaxed states over various interaction vertices. Smaller distances correspond to better recall and hence better a better associative memory.}
    \label{Fig:OriginalModelHyperparameterSearch}
\end{figure}

\begin{figure}[H]
    \centering
    \begin{subfigure}[t]{0.48\textwidth}
        \includegraphics[width=\textwidth]{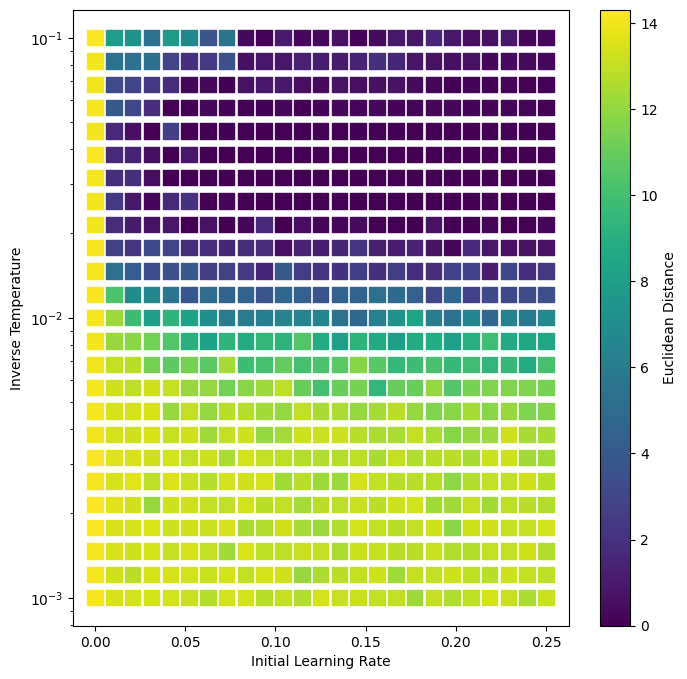}
        \caption{\(n=2\)}
    \end{subfigure}
    \begin{subfigure}[t]{0.48\textwidth}
        \includegraphics[width=\textwidth]{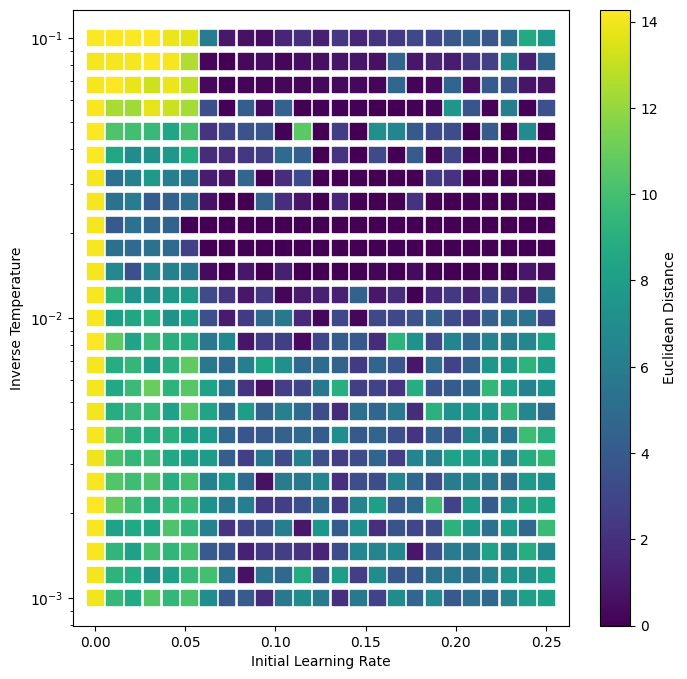}
        \caption{\(n=5\)}
    \end{subfigure}
    \hfill
    \begin{subfigure}[t]{0.48\textwidth}
        \includegraphics[width=\textwidth]{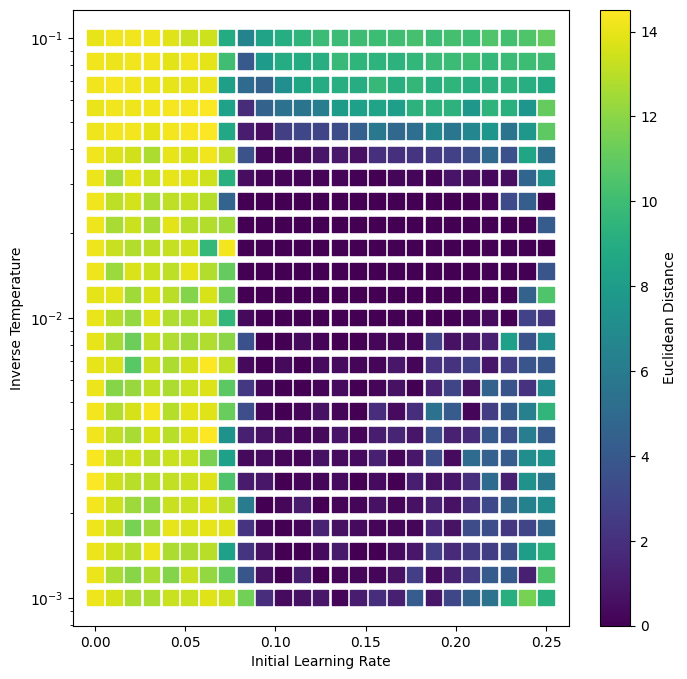}
        \caption{\(n=10\)}
    \end{subfigure}
    \begin{subfigure}[t]{0.48\textwidth}
        \includegraphics[width=\textwidth]{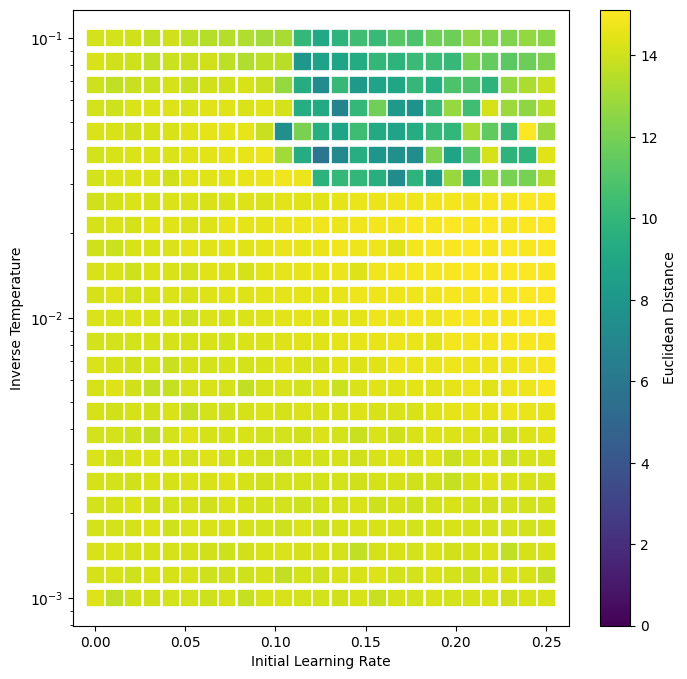}
        \caption{\(n=20\)}
    \end{subfigure}
    
    \caption{Fine hyperparameter search space for the original network, measuring the Euclidean distance between learned states and relaxed states over various interaction vertices. Smaller distances correspond to better recall and hence better a better associative memory.}
\end{figure}

Figure \ref{Fig:OriginalModelHyperparameterSearch} shows the hyperparameters across networks with various interaction vertices. The color of the heat map shows how far the relaxed states are from the original learned states, with lower / darker values being better. The optimal region --- that is, the combination of hyperparameters that give low distances --- is somewhere around \(\frac{1}{T} \approx 10^{-2}\) and learning rate \(\approx 0.5\) but shifts considerably with the interaction vertex. We also find a significant increase in distance with the learning rate; higher learning rates tend to degrade network performance. At even modest interaction vertices we find the optimal region is fleeting enough to not appear in our grid search. It is tempting to claim that a finer grid search may reveal the region to persist. However, closer inspection of Figure \ref{Fig:OriginalModelHyperparameterSearchN20} shows that not only has the optimal region vanished at this granularity, but the same region has \textit{increased} the distance measure compared to its surroundings. Even if the optimal region exists and is very small, it is apparently surrounded by an increasingly suboptimal region. This is troublesome and makes working with the network difficult.

Note that we have avoided floating point overflow by engineering our experiments to remain within the bounds of a double. In general, this network would fail to train for larger interaction vertices or data dimensions. However, the performance degradation seen at larger interaction vertices in Figure \ref{Fig:OriginalModelHyperparameterSearch} is not due to floating point overflow.

\subsection{Modified Network Hyperparameter Results}
\label{Section:ModifiedNetworkAutoassociativeTask}

\begin{figure}[H]
    \centering
    \begin{subfigure}[t]{0.48\textwidth}
        \includegraphics[width=\textwidth]{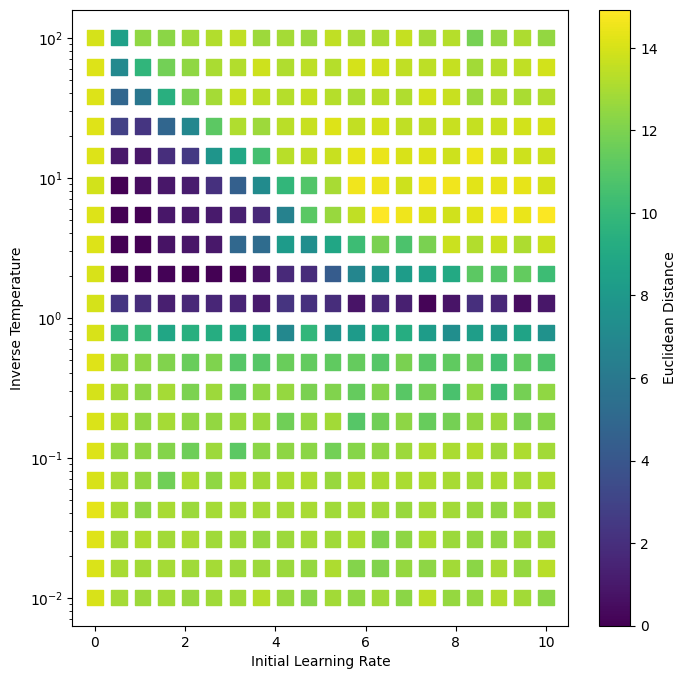}
        \caption{\(n=2\)}
    \end{subfigure}
    \begin{subfigure}[t]{0.48\textwidth}
        \includegraphics[width=\textwidth]{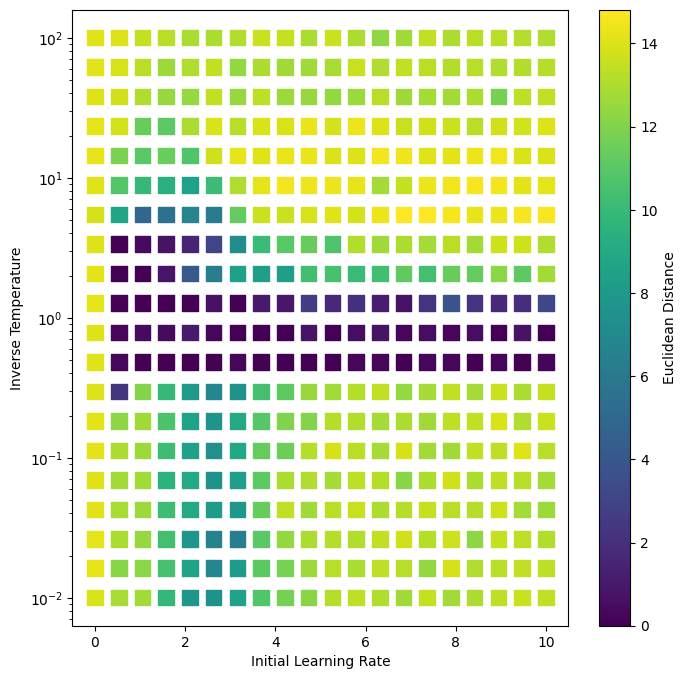}
        \caption{\(n=5\)}
    \end{subfigure}
    \hfill
    \begin{subfigure}[t]{0.48\textwidth}
        \includegraphics[width=\textwidth]{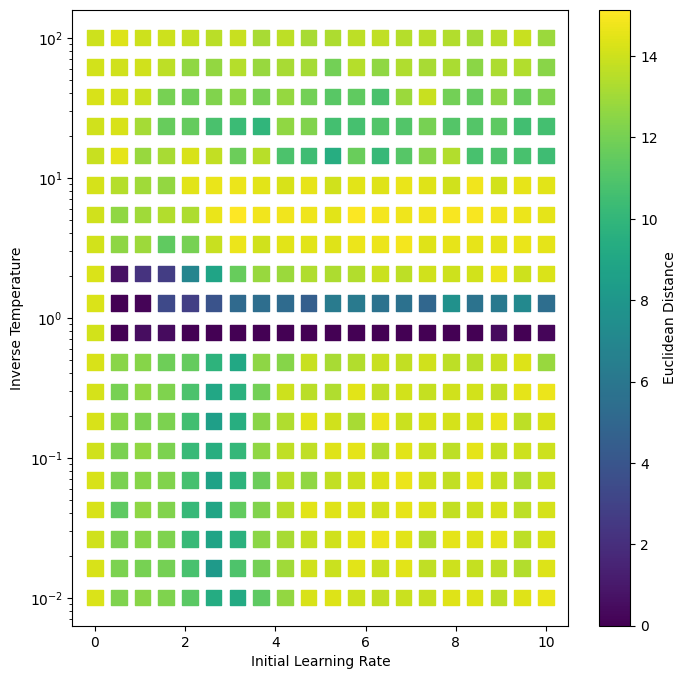}
        \caption{\(n=10\)}
    \end{subfigure}
    \begin{subfigure}[t]{0.48\textwidth}
        \includegraphics[width=\textwidth]{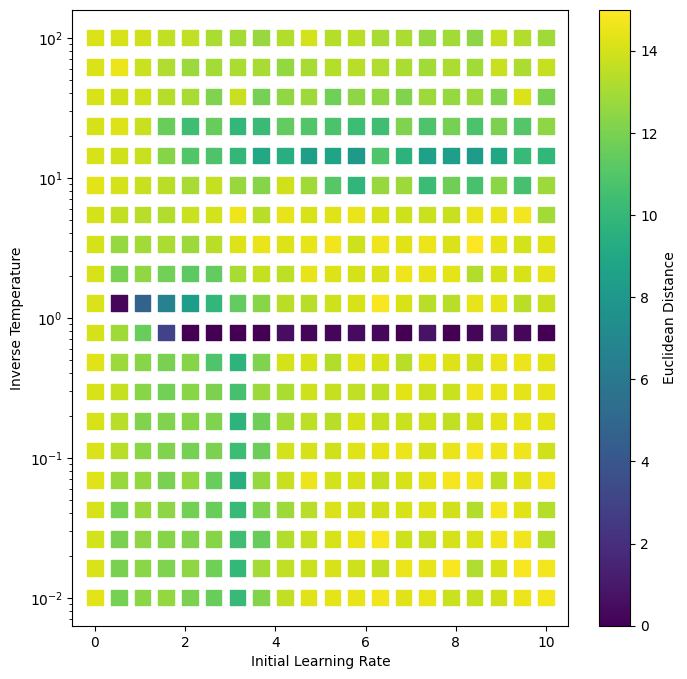}
        \caption{\(n=20\)}
    \end{subfigure}
    
    \caption{Coarse hyperparameter search space for the modified network, measuring the Euclidean distance between learned states and relaxed states over various interaction vertices. Smaller distances correspond to better recall and hence better a better associative memory.}
    \label{Fig:ModifiedModelHyperparameterSearch}
\end{figure}

\begin{figure}[H]
    \centering
    \begin{subfigure}[t]{0.48\textwidth}
        \includegraphics[width=\textwidth]{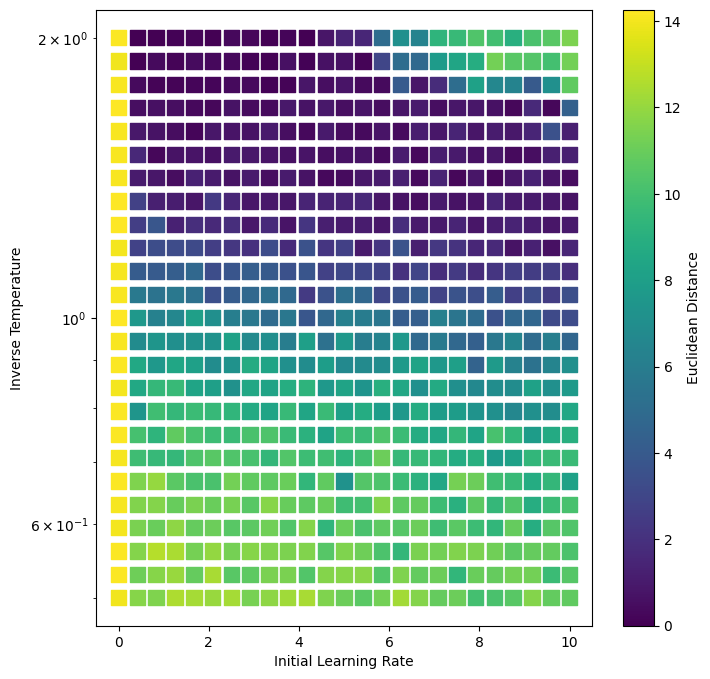}
        \caption{\(n=2\)}
    \end{subfigure}
    \begin{subfigure}[t]{0.48\textwidth}
        \includegraphics[width=\textwidth]{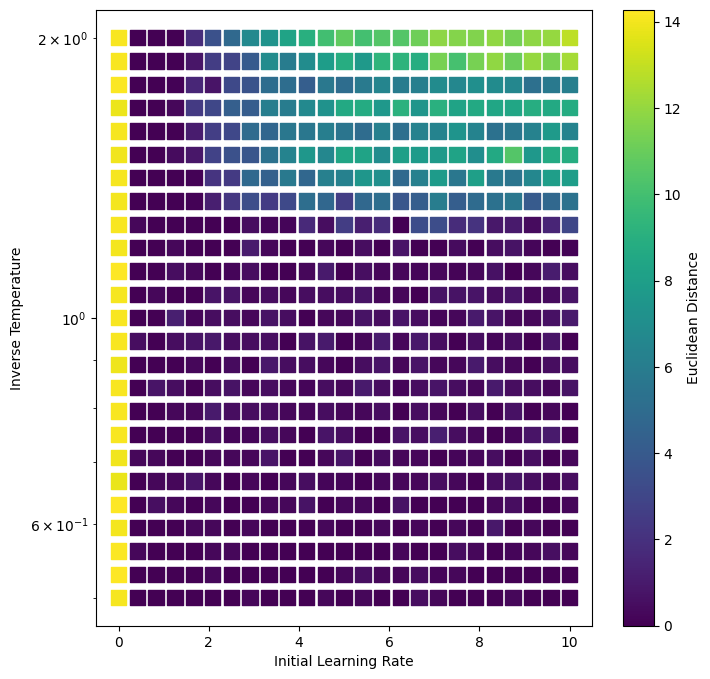}
        \caption{\(n=5\)}
    \end{subfigure}
    \hfill
    \begin{subfigure}[t]{0.48\textwidth}
        \includegraphics[width=\textwidth]{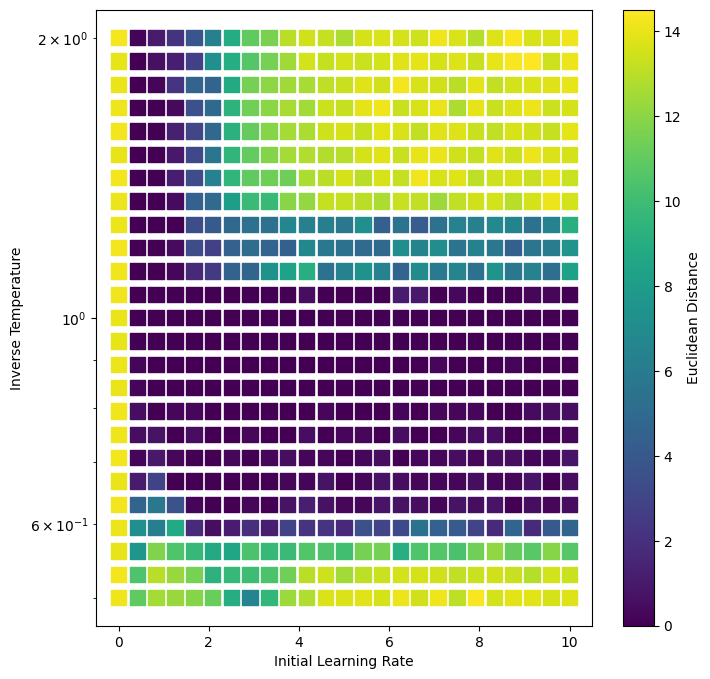}
        \caption{\(n=10\)}
    \end{subfigure}
    \begin{subfigure}[t]{0.48\textwidth}
        \includegraphics[width=\textwidth]{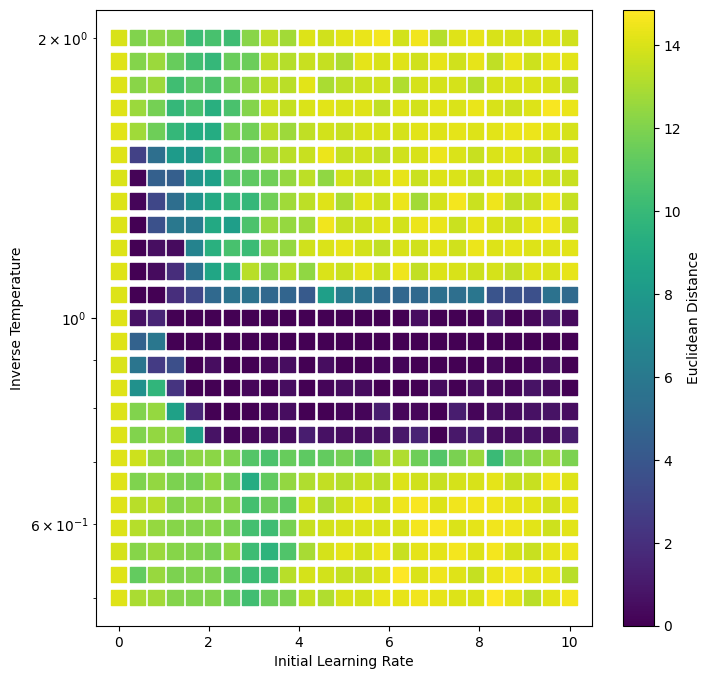}
        \caption{\(n=20\)}
    \end{subfigure}
    
    \caption{Fine hyperparameter search space for the modified network, measuring the Euclidean distance between learned states and relaxed states over various interaction vertices. Smaller distances correspond to better recall and hence better a better associative memory.}
    \label{Fig:ModifiedModelHyperparameterSearchTight}
\end{figure}

Figure \ref{Fig:ModifiedModelHyperparameterSearch} shows the same hyperparameter search for our modified network. Note that we have shifted the scale of the inverse temperature as discussed in Section \ref{Section:ModificationFormalization}. As in Figure \ref{Fig:OriginalModelHyperparameterSearch} we find the optimal region shifts slightly for small interaction vertices (\(n \leq 5\)) but unlike the original network we find the region stabilizes and remains substantial for large interaction vertices. Figure \ref{Fig:ModifiedModelHyperparameterSearchTight} shows a finer grid search over the region of interest, showing the optimal region stabilizes around \(\frac{1}{T} \approx 9 \cdot 10^{-1}\). We find it is common across many network dimensions and task sizes for the inverse temperature to stabilize near \(1\), making it much easier to tune the hyperparameters of the Dense Associative Memory. The optimal region also extends to much larger initial learning rates than it does in the original network. Most notably, we find that the optimal region's position remains stable and size remains large across many values of the interaction vertex for the same network dimension.

\subsection{MNIST Classification}

So far, we have focused on the Dense Associative Memory as an autoassociative memory, where all neurons are updated at each step and may be updated numerous times until the state reaches stability. Another, perhaps more popular use case of the network in current literature is as a classifier. By splitting the memory vectors into two parts --- a section for input data and a section for classes as logits --- the network can be run as a classifier by only updating the classification neurons, and only updating those neurons once \citep{KrotovHopfield2016}. Krotov and Hopfield also found it was necessary to leave the weights of the classification neurons unclamped, unlike the input data section which remained clamped between \(-1\) and \(1\). This is another step away from traditional autoassociative memories, but the resulting network is still worth investigating with our modifications due to its popularity. \citet{KrotovHopfield2016} show an equivalence between the Dense Associative Memory operating in this mode and a single-hidden-layer feed-forward neural network by taking the Taylor expansion of the similarity score calculation and ignoring some crosstalk terms. In doing this, the value of \(\beta\) was also set to cancel some constants from the Taylor expansion. It is difficult to say how much of the literature is using the autoassociative memory model compared to the feed-forward equivalent, however we suspect that the feed-forward equivalent effectively implements some of our modifications (namely, normalizing the value \(\beta\) by the network dimension) which may explain the popularity of this mode, as the network is more stable. 

In our results below, we have trained the Dense Associative Memory on the MNIST dataset and note the validation F1 score across hyperparameter space. We have used the autoassociative memory model, rather than the feed-forward equivalent. This means we have \textit{not} explicitly ignored the effects of the classification neurons on one another, as is done in constructing the feed-forward equivalent, although the effect is likely negligible. Note that we have significantly different scales for the original and modified network's values of \(\beta\), which is not seen in the previous results. We believe this is due to leaving some memory weights unclamped, as well as only updating a small number of neurons as required for classification, but again the optimal value of the inverse temperature stabilizes around \(1\) for our modified network. Notably, our range of \(\beta\) for the original network matches the range found by \citep{KrotovHopfield2016}. In all experiments we trained the network for 500 epochs with 256 memory vectors.

\begin{figure}[H]
    \centering
    \begin{subfigure}[t]{0.48\textwidth}
        \includegraphics[width=\textwidth]{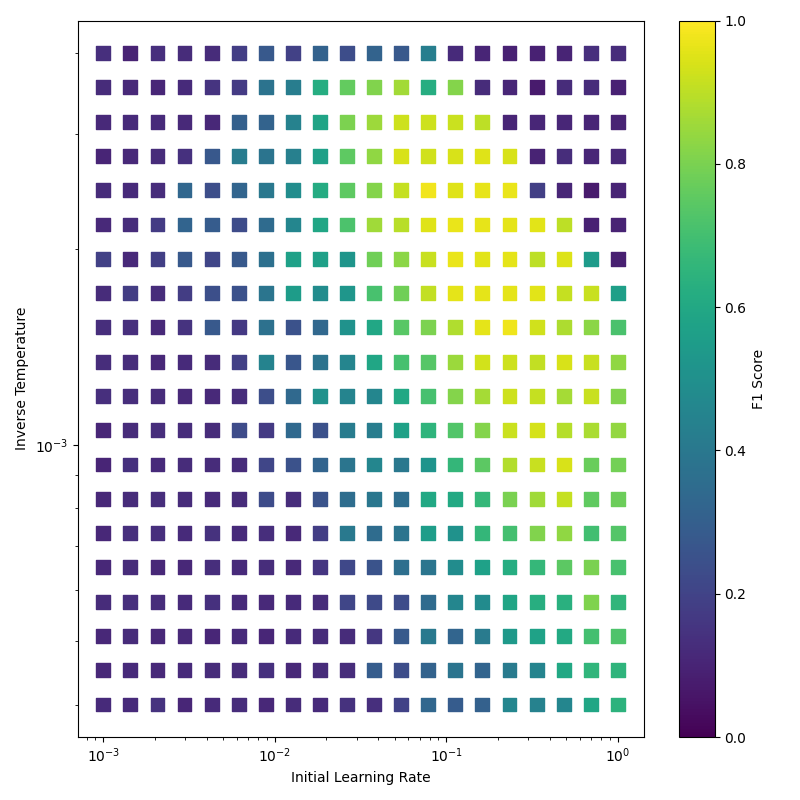}
        \caption{\(n=2\)}
    \end{subfigure}
    \begin{subfigure}[t]{0.48\textwidth}
        \includegraphics[width=\textwidth]{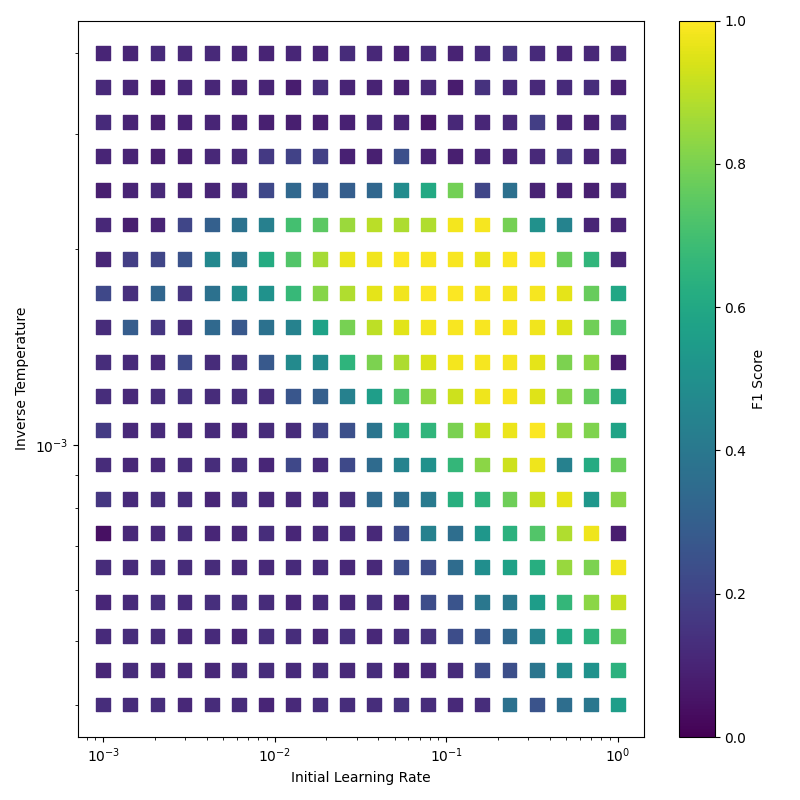}
        \caption{\(n=5\)}
    \end{subfigure}
    \hfill
    \begin{subfigure}[t]{0.48\textwidth}
        \includegraphics[width=\textwidth]{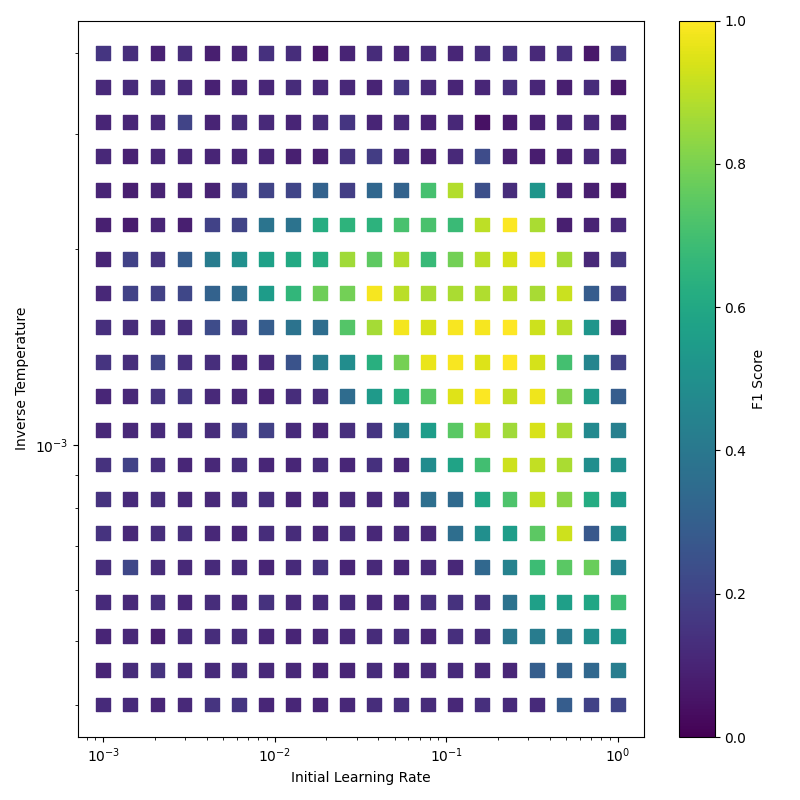}
        \caption{\(n=10\)}
    \end{subfigure}
    \begin{subfigure}[t]{0.48\textwidth}
        \includegraphics[width=\textwidth]{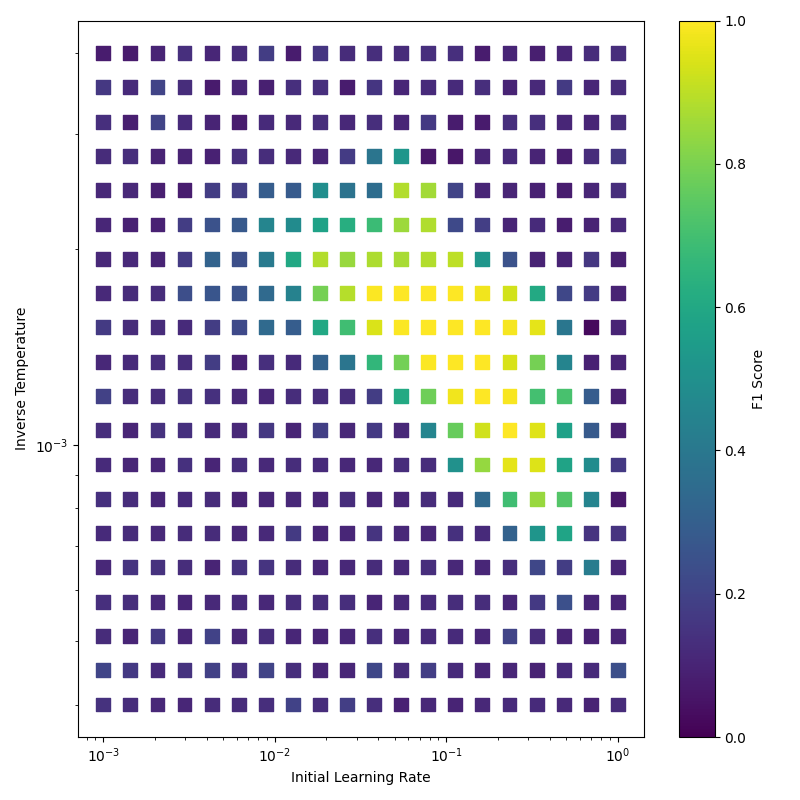}
        \caption{\(n=20\)}
    \end{subfigure}
    \caption{Hyperparameter search space for the original network, measuring the validation F1 score on the MNIST dataset. A larger F1 score corresponds to a better performing network.}
    \label{Fig:OriginalModelMNISTHyperparameterSearch}
\end{figure}

\begin{figure}[H]
    \centering
    \begin{subfigure}[t]{0.48\textwidth}
        \includegraphics[width=\textwidth]{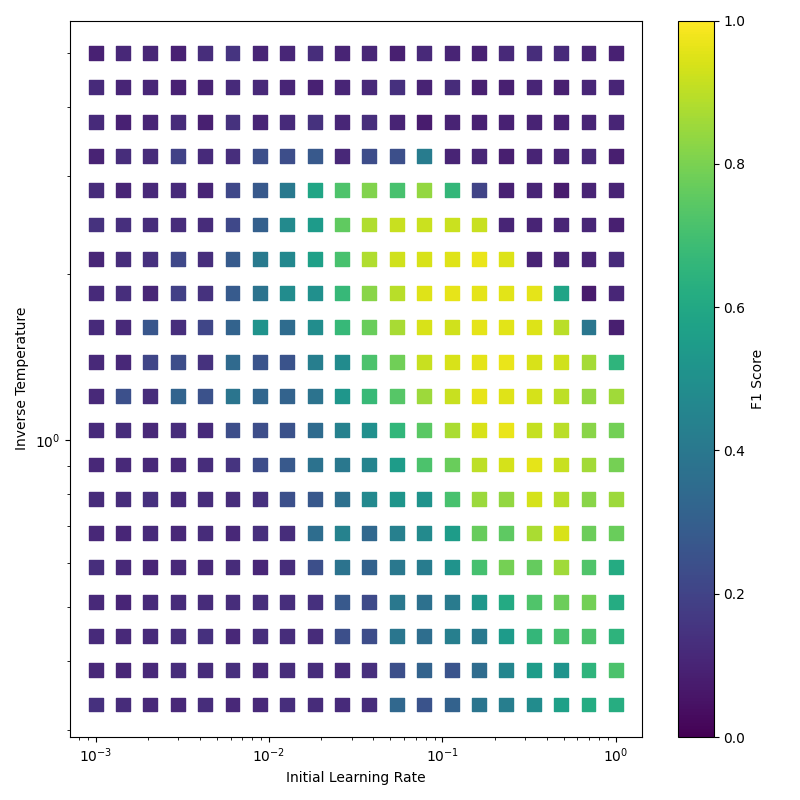}
        \caption{\(n=2\)}
    \end{subfigure}
    \begin{subfigure}[t]{0.48\textwidth}
        \includegraphics[width=\textwidth]{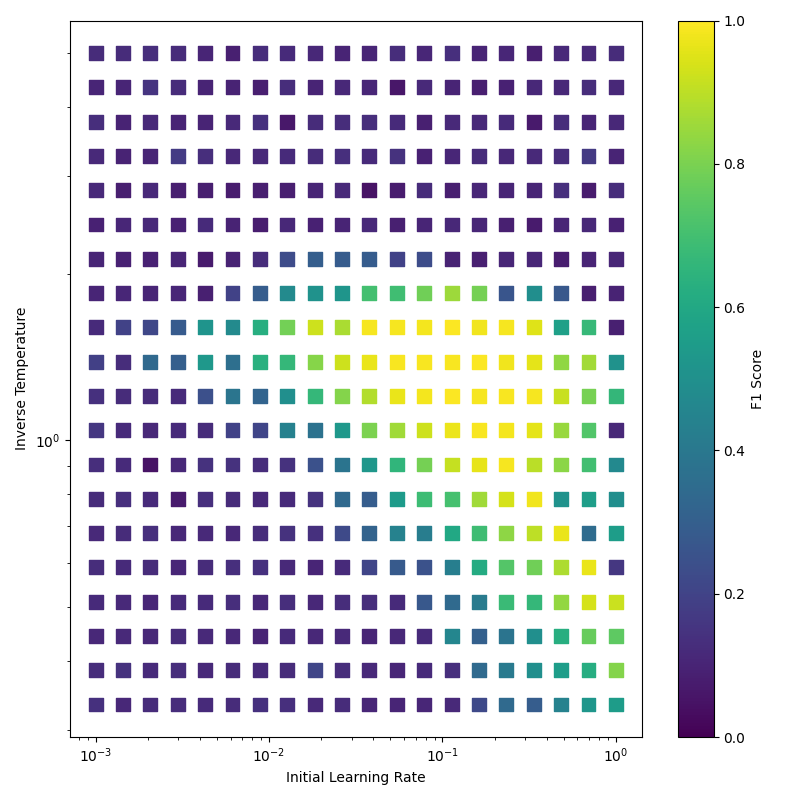}
        \caption{\(n=5\)}
    \end{subfigure}
    \hfill
    \begin{subfigure}[t]{0.48\textwidth}
        \includegraphics[width=\textwidth]{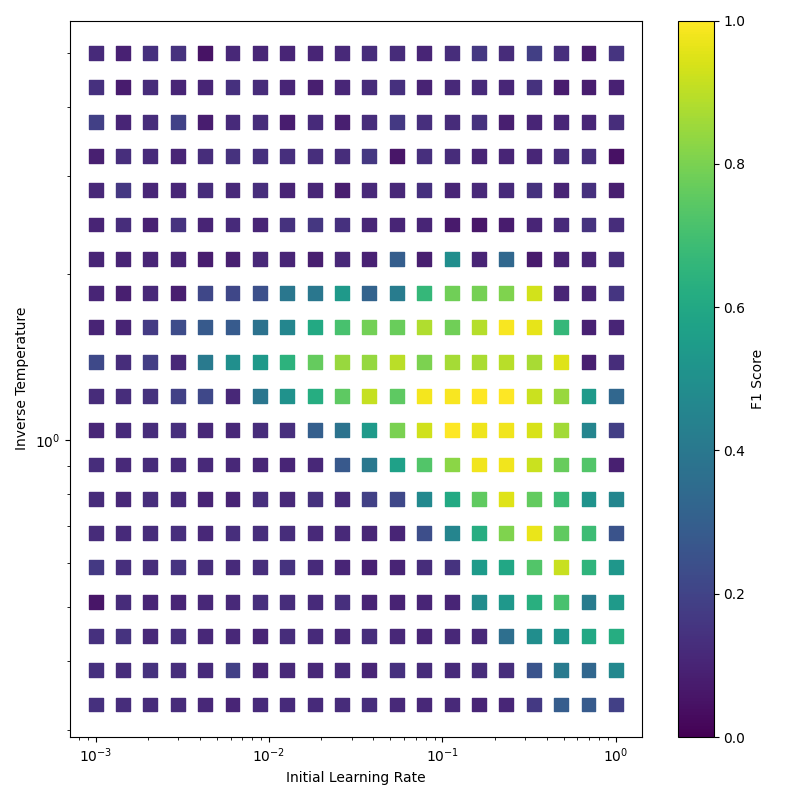}
        \caption{\(n=10\)}
    \end{subfigure}
    \begin{subfigure}[t]{0.48\textwidth}
        \includegraphics[width=\textwidth]{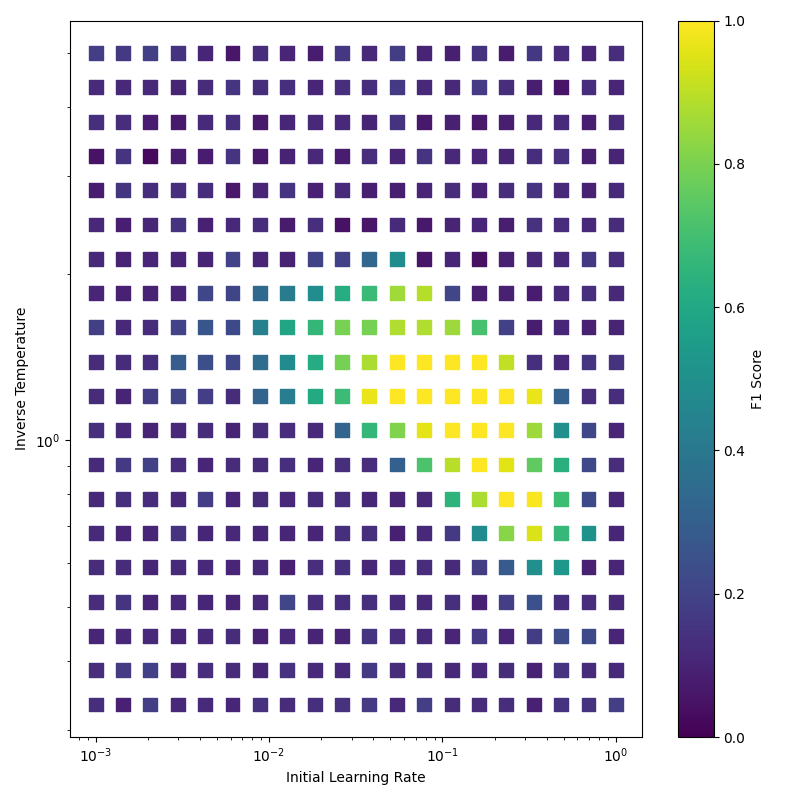}
        \caption{\(n=20\)}
    \end{subfigure}
    \caption{Hyperparameter search space for the modified network, measuring the validation F1 score on the MNIST dataset. A larger F1 score corresponds to a better performing network.}
    \label{Fig:ModifiedModelMNISTHyperparameterSearch}
\end{figure}

Figures \ref{Fig:OriginalModelMNISTHyperparameterSearch} and \ref{Fig:ModifiedModelMNISTHyperparameterSearch} show the optimal hyperparameter region for the original and modified network on classifying the MNIST dataset. Note that in these figures, we want a higher F1 score and hence the yellow region is better, unlike previous figures where we wanted a lower Euclidean distance and hence the purple region was better. In both Figure \ref{Fig:OriginalModelMNISTHyperparameterSearch} and \ref{Fig:ModifiedModelMNISTHyperparameterSearch} the initially large region for \(n=2\) shrinks slightly as \(n\) grows, and appears to shrink by proportionally the same amount in both the original and modified network. The shape of the region also remains consistent in both the modified and unmodified networks. This indicates that our modifications have preserved the stability in classification based tasks. Our modifications have, however, shifted the optimal region to \(\beta \approx 1\), meaning we have a better idea of where to search for optimal hyperparameters. While not as significant a result as in Section \ref{Section:ModifiedNetworkAutoassociativeTask} and \ref{Section:ModifiedNetworkAutoassociativeTask}, this result is still useful in working with the Dense Associative Memory as the location of the optimal hyperparameter region is consistent across different datasets and tasks.

\section{Conclusion}

In this work, we have investigated the technical details of the Dense Associative Memory and its implementation. We note that the original network specification leads to floating point imprecision and overflow when calculating intermediate values for both update and learning. We provide details on when this imprecision occurs and show the conditions are more likely when the interaction vertex is large based on the feature-to-prototype transition of the memory vectors \citep{KrotovHopfield2016}. We propose a modification to the network implementation that prevents the floating point issues. We prove our modifications do not alter the network properties, such as the capacity and autoassociative nature. Our proof relies on the interaction function being homogenous, however this property is stronger than is required, and we find empirically that some nonhomogenous functions also give well-behaved Dense Associative Memories. We then show our modified network has optimal hyperparameter regions that do not shift based on the choice of interaction vertex for purely autoassociative tasks. For classification like tasks, such as MNIST classification, our modifications do not appear to radically improve the optimal hyperparameter region but rather shift the region to a common location that makes tuning the network easier. Our modifications greatly simplify working with the Dense Associative Memory, as experiments on a dataset do not need to search across a potentially large hyperparameter space for each change in the interaction vertex. We also find several hyperparameters do not need tuning in our experiments, hinting at a potentially simpler network that is easier to tune and interpret.

\printbibliography

\appendix
\section{Full Results of Hyperparameter Searches}
\label{Section:AppendixResults}

\subsection{Original Network, Dimension 100}

These results are from the original network, have dimension 100, and train on 20 learned states.

\begin{figure}[H]
    \centering
    \begin{subfigure}[t]{0.48\textwidth}
        \includegraphics[width=\textwidth]{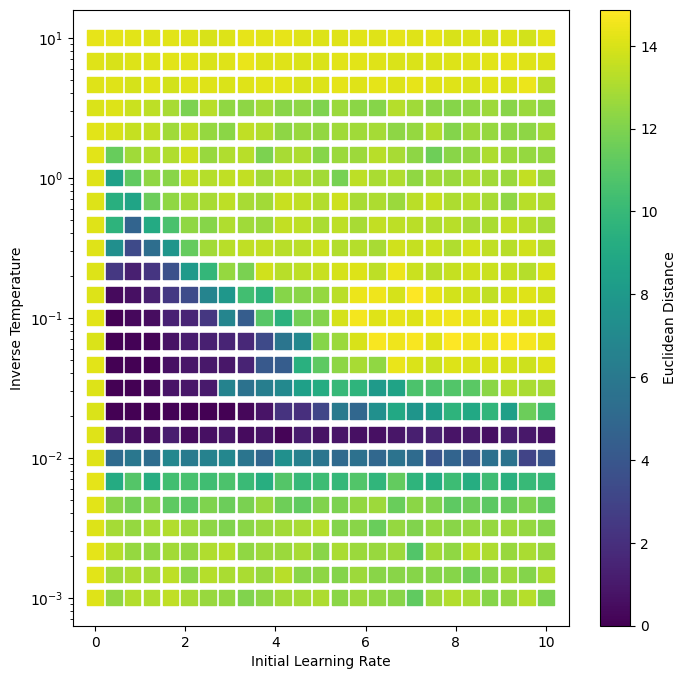}
        \caption{Coarse search space}
    \end{subfigure}
    \begin{subfigure}[t]{0.48\textwidth}
        \includegraphics[width=\textwidth]{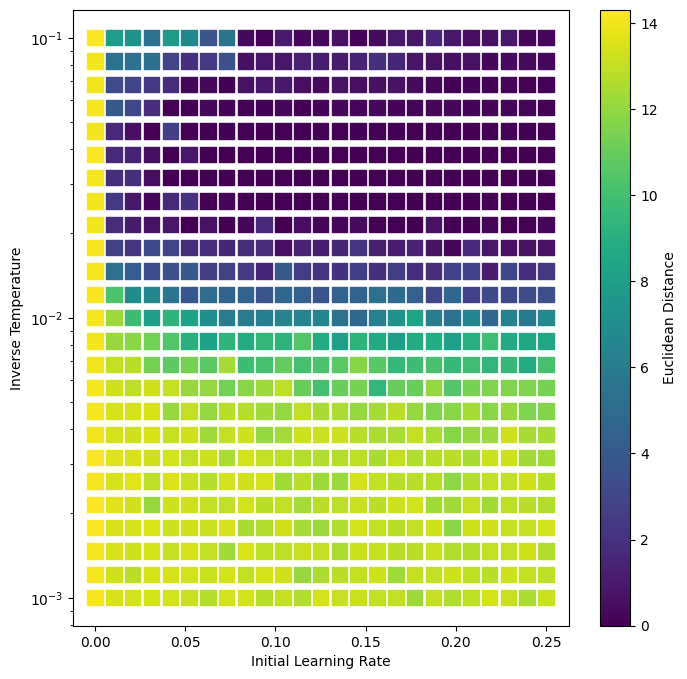}
        \caption{Fine search space}
    \end{subfigure}
    \caption{Hyperparameter search space for \(n=2\)}
\end{figure}

\begin{figure}[H]
    \centering
    \begin{subfigure}[t]{0.48\textwidth}
        \includegraphics[width=\textwidth]{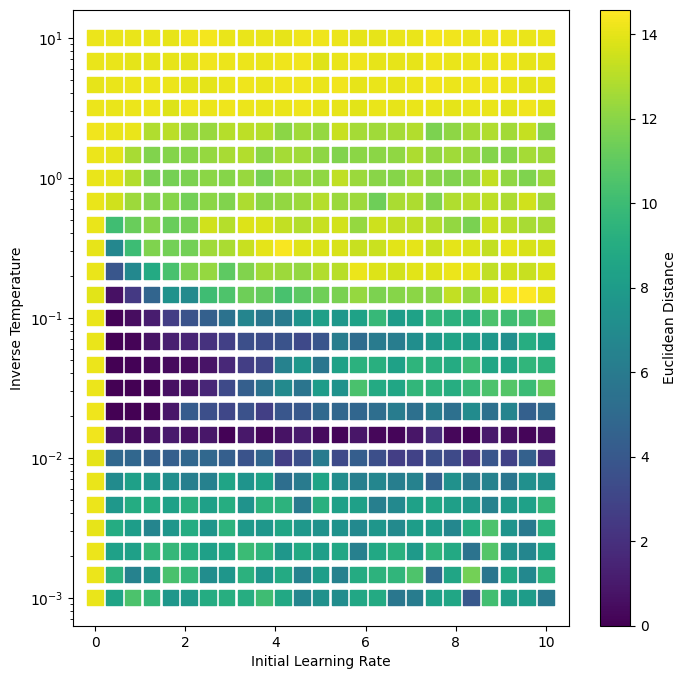}
        \caption{Coarse search space}
    \end{subfigure}
    \begin{subfigure}[t]{0.48\textwidth}
        \includegraphics[width=\textwidth]{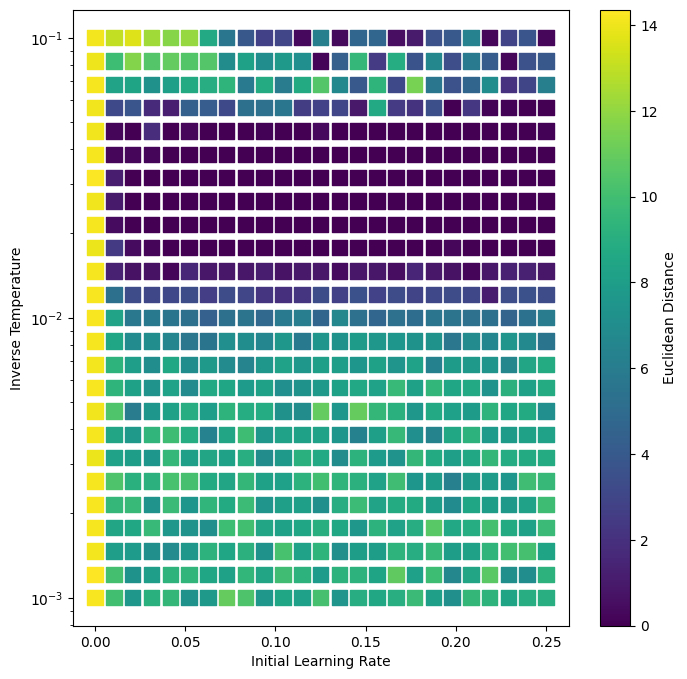}
        \caption{Fine search space}
    \end{subfigure}
    \caption{Hyperparameter search space for \(n=3\)}
\end{figure}

\begin{figure}[H]
    \centering
    \begin{subfigure}[t]{0.48\textwidth}
        \includegraphics[width=\textwidth]{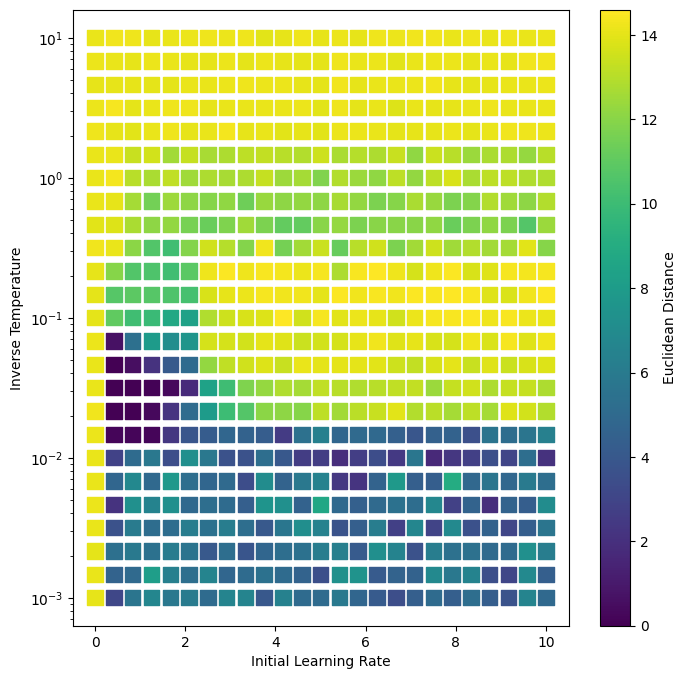}
        \caption{Coarse search space}
    \end{subfigure}
    \begin{subfigure}[t]{0.48\textwidth}
        \includegraphics[width=\textwidth]{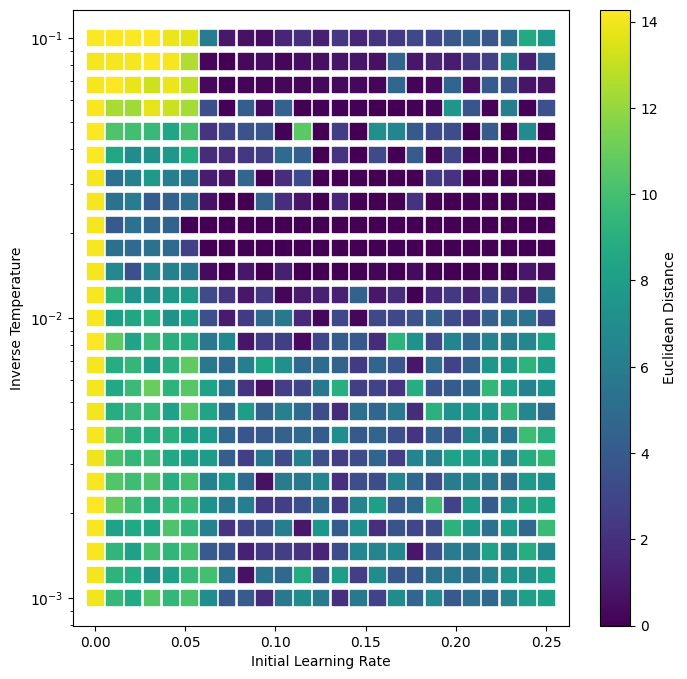}
        \caption{Fine search space}
    \end{subfigure}
    \caption{Hyperparameter search space for \(n=5\)}
\end{figure}

\begin{figure}[H]
    \centering
    \begin{subfigure}[t]{0.48\textwidth}
        \includegraphics[width=\textwidth]{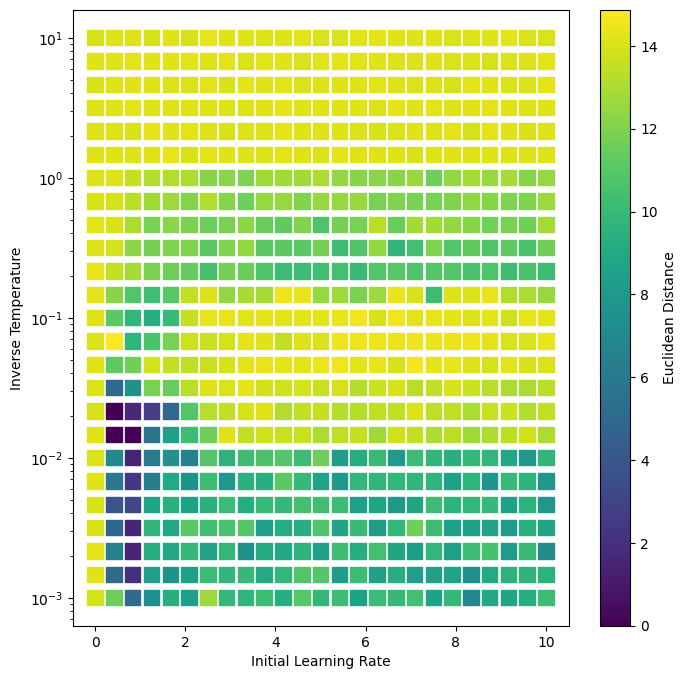}
        \caption{Coarse search space}
    \end{subfigure}
    \begin{subfigure}[t]{0.48\textwidth}
        \includegraphics[width=\textwidth]{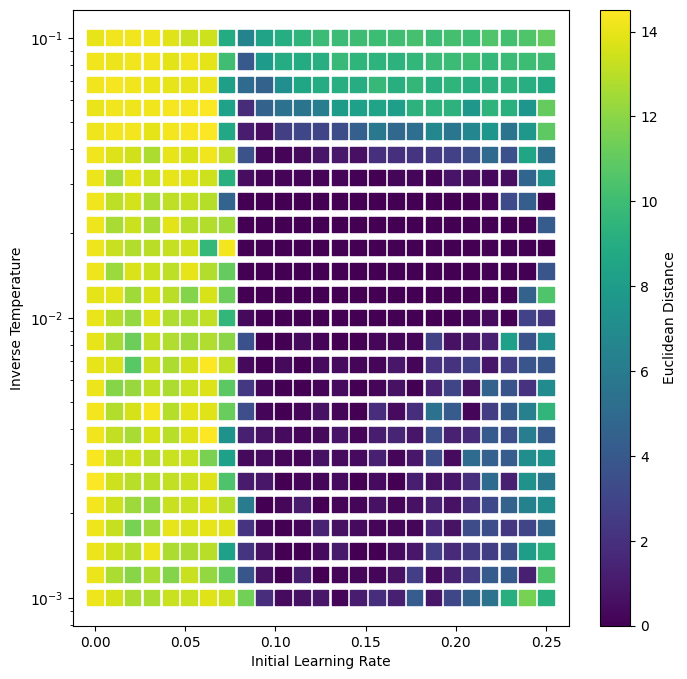}
        \caption{Fine search space}
    \end{subfigure}
    \caption{Hyperparameter search space for \(n=10\)}
\end{figure}

\begin{figure}[H]
    \centering
    \begin{subfigure}[t]{0.48\textwidth}
        \includegraphics[width=\textwidth]{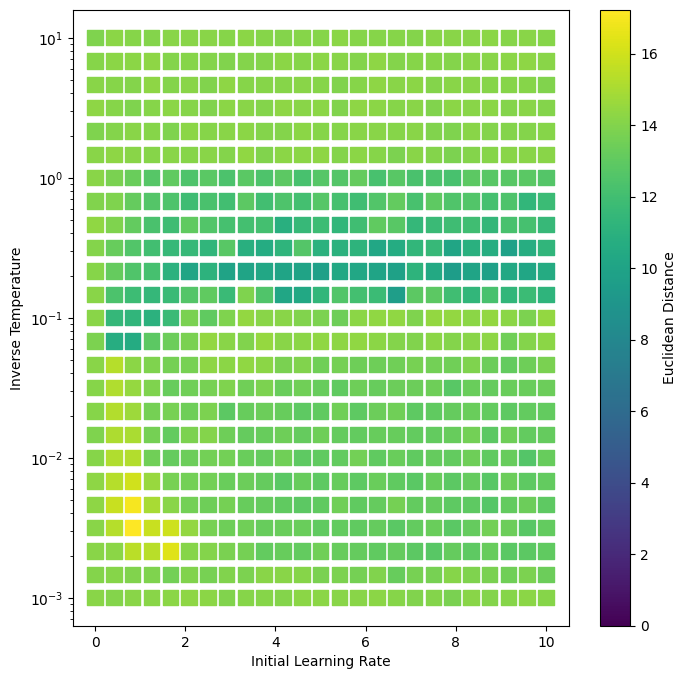}
        \caption{Coarse search space}
    \end{subfigure}
    \begin{subfigure}[t]{0.48\textwidth}
        \includegraphics[width=\textwidth]{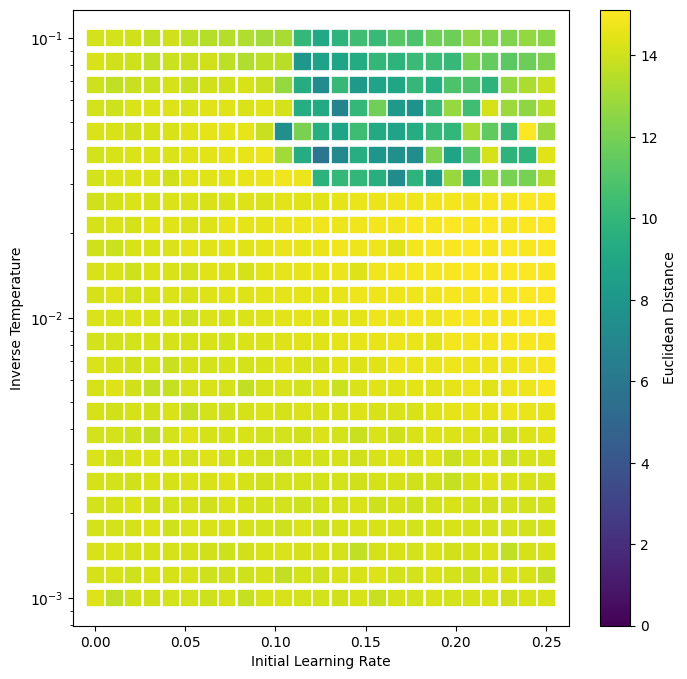}
        \caption{Fine search space}
    \end{subfigure}
    \caption{Hyperparameter search space for \(n=20\)}
\end{figure}

\begin{figure}[H]
    \centering
    \begin{subfigure}[t]{0.48\textwidth}
        \includegraphics[width=\textwidth]{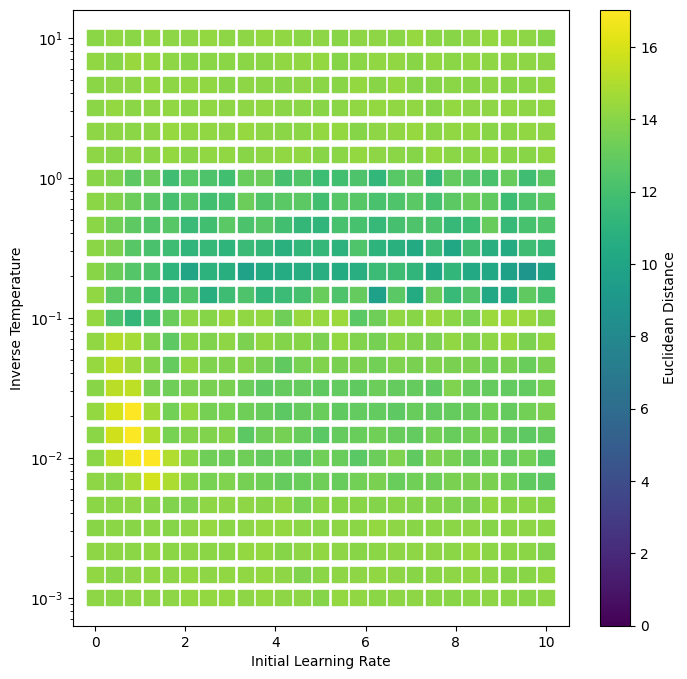}
        \caption{Coarse search space}
    \end{subfigure}
    \begin{subfigure}[t]{0.48\textwidth}
        \includegraphics[width=\textwidth]{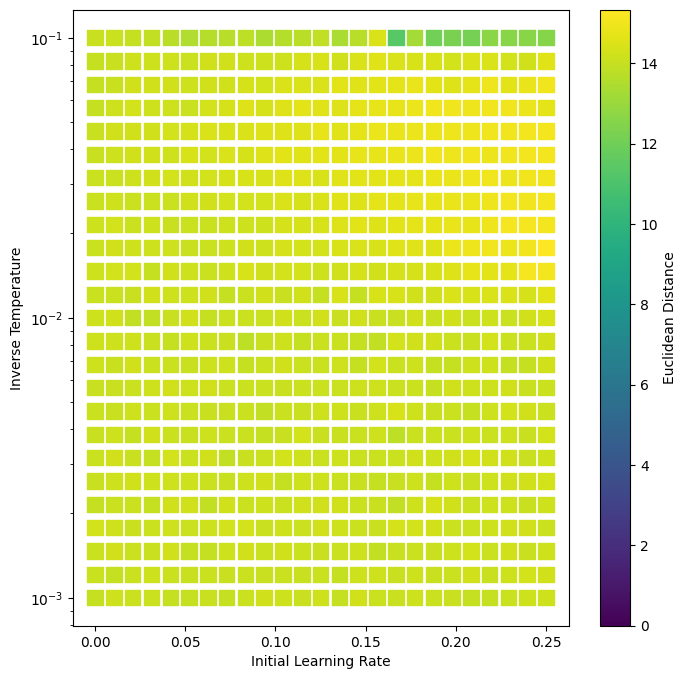}
        \caption{Fine search space}
    \end{subfigure}
    \caption{Hyperparameter search space for \(n=30\)}
\end{figure}

\subsection{Modified Network, Dimension 100}
\label{Appendix:ModifiedDim100}

These results are from our modified network, have dimension 100, and train on 20 learned states.

\begin{figure}[H]
    \centering
    \begin{subfigure}[t]{0.48\textwidth}
        \includegraphics[width=\textwidth]{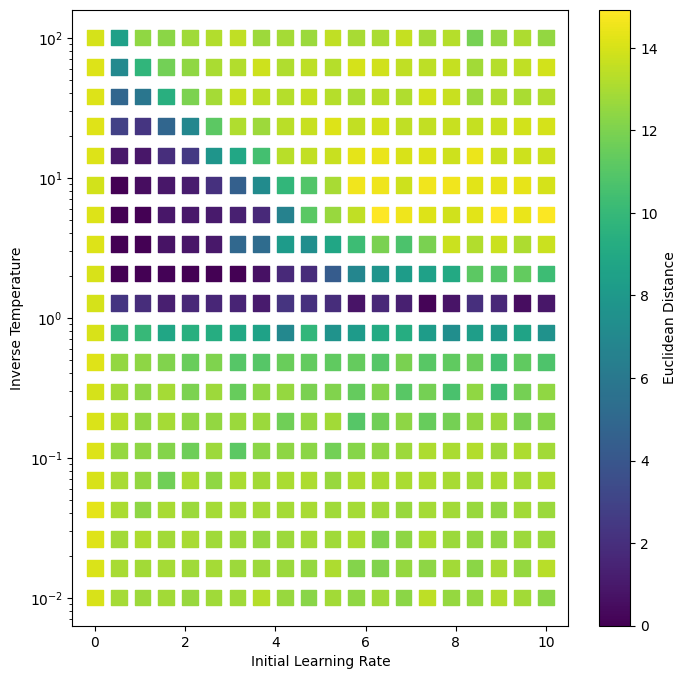}
        \caption{Coarse search space}
    \end{subfigure}
    \begin{subfigure}[t]{0.48\textwidth}
        \includegraphics[width=\textwidth]{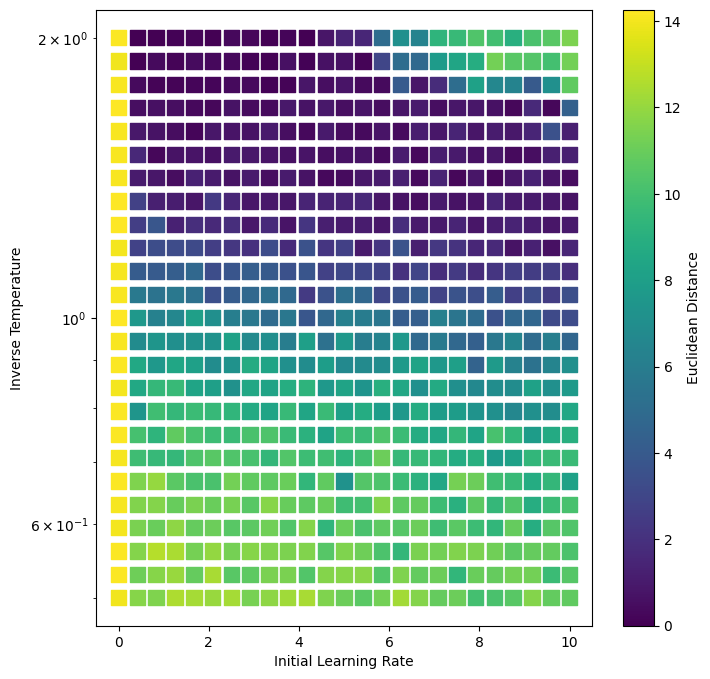}
        \caption{Fine search space}
    \end{subfigure}
    \caption{Hyperparameter search space for \(n=2\)}
\end{figure}

\begin{figure}[H]
    \centering
    \begin{subfigure}[t]{0.48\textwidth}
        \includegraphics[width=\textwidth]{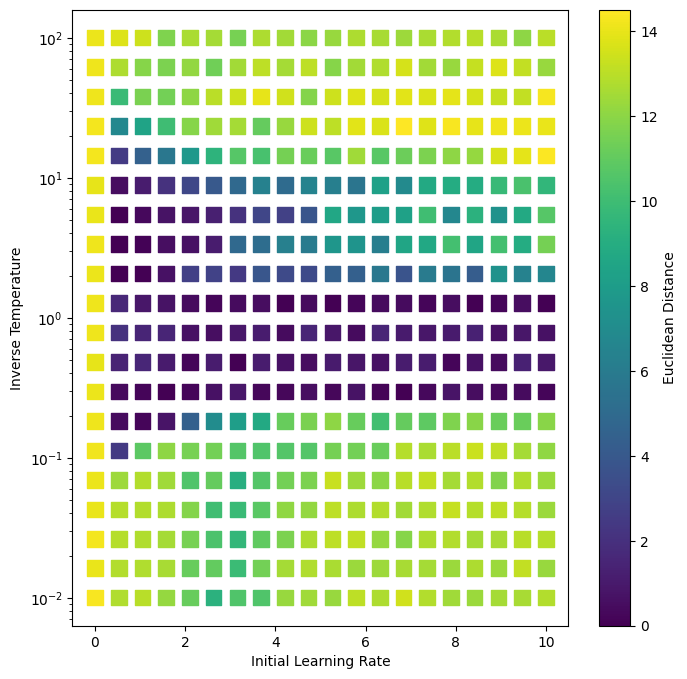}
        \caption{Coarse search space}
    \end{subfigure}
    \begin{subfigure}[t]{0.48\textwidth}
        \includegraphics[width=\textwidth]{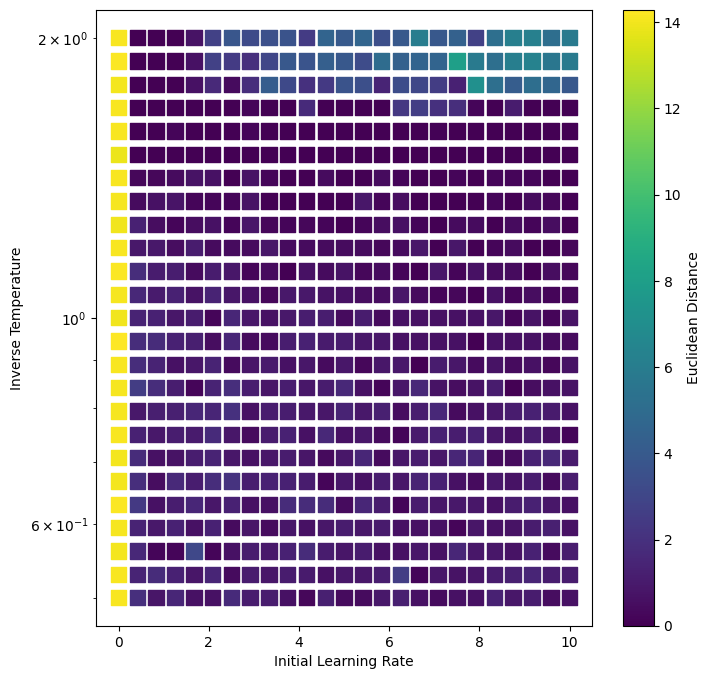}
        \caption{Fine search space}
    \end{subfigure}
    \caption{Hyperparameter search space for \(n=3\)}
\end{figure}

\begin{figure}[H]
    \centering
    \begin{subfigure}[t]{0.48\textwidth}
        \includegraphics[width=\textwidth]{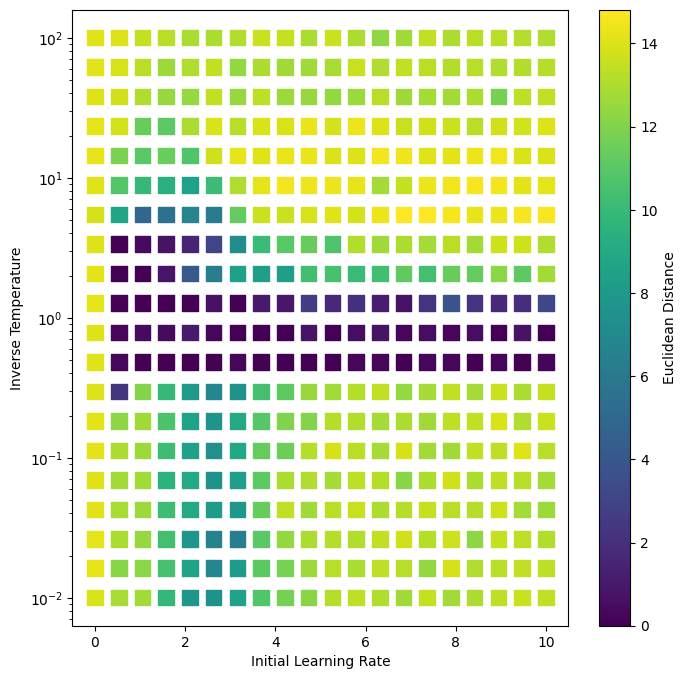}
        \caption{Coarse search space}
    \end{subfigure}
    \begin{subfigure}[t]{0.48\textwidth}
        \includegraphics[width=\textwidth]{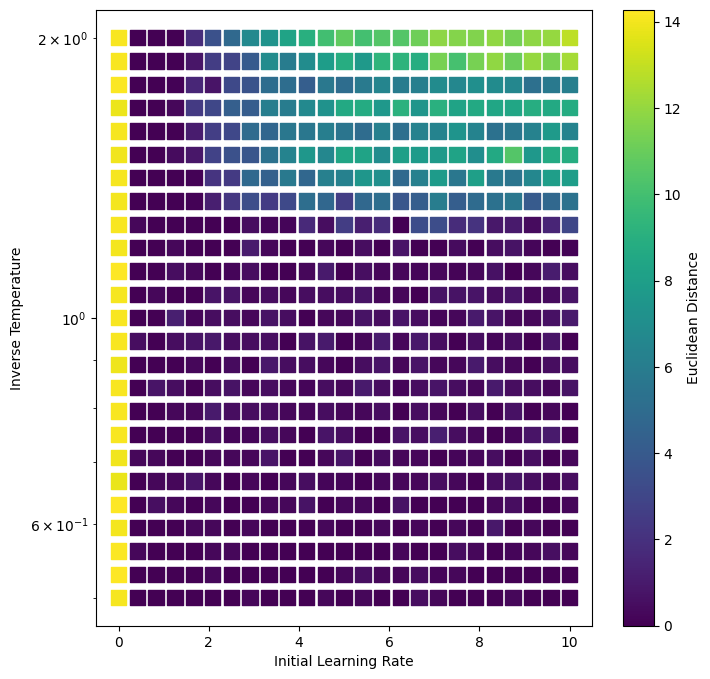}
        \caption{Fine search space}
    \end{subfigure}
    \caption{Hyperparameter search space for \(n=5\)}
\end{figure}

\begin{figure}[H]
    \centering
    \begin{subfigure}[t]{0.48\textwidth}
        \includegraphics[width=\textwidth]{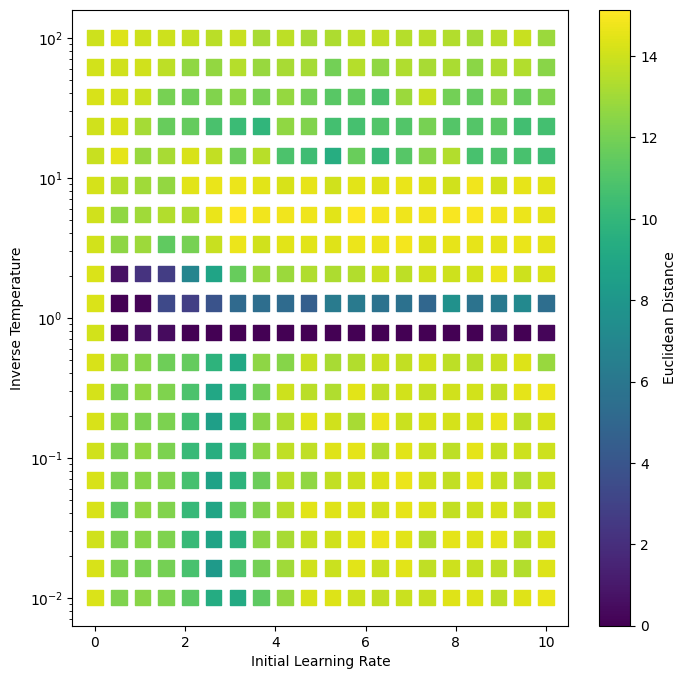}
        \caption{Coarse search space}
    \end{subfigure}
    \begin{subfigure}[t]{0.48\textwidth}
        \includegraphics[width=\textwidth]{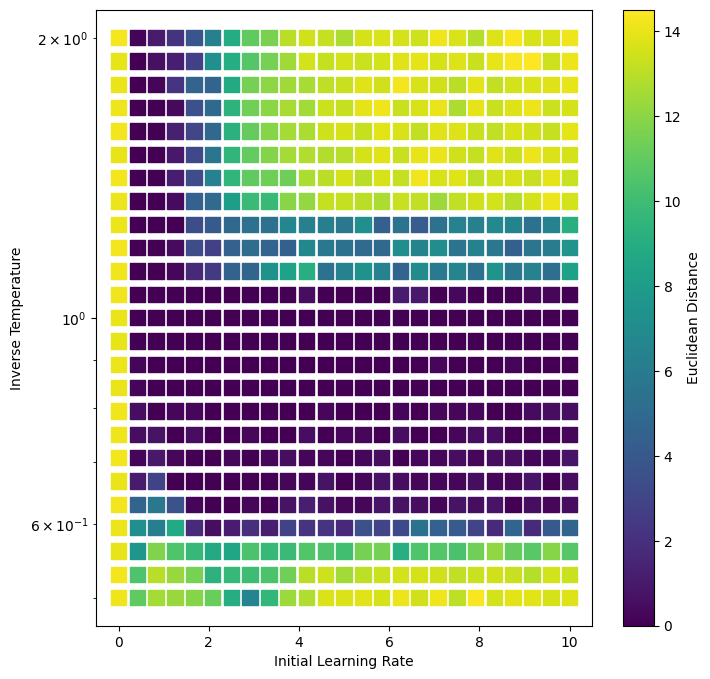}
        \caption{Fine search space}
    \end{subfigure}
    \caption{Hyperparameter search space for \(n=10\)}
\end{figure}

\begin{figure}[H]
    \centering
    \begin{subfigure}[t]{0.48\textwidth}
        \includegraphics[width=\textwidth]{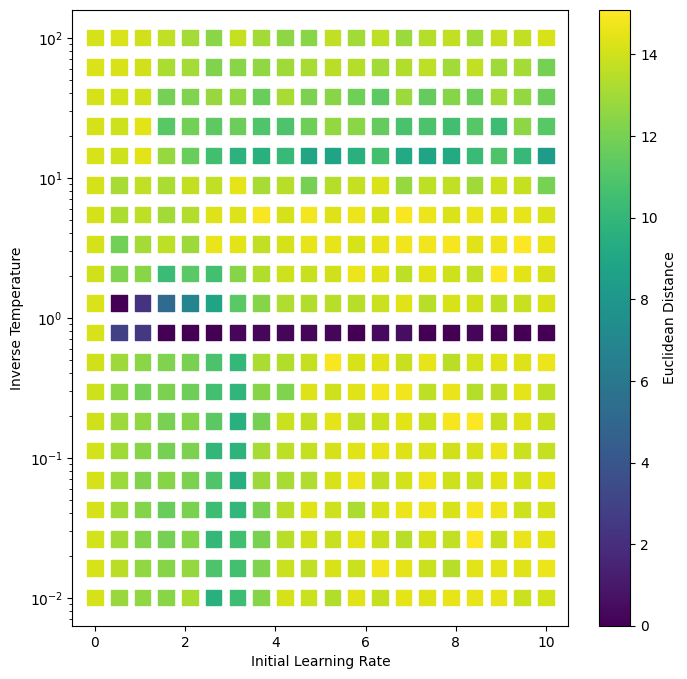}
        \caption{Coarse search space}
    \end{subfigure}
    \begin{subfigure}[t]{0.48\textwidth}
        \includegraphics[width=\textwidth]{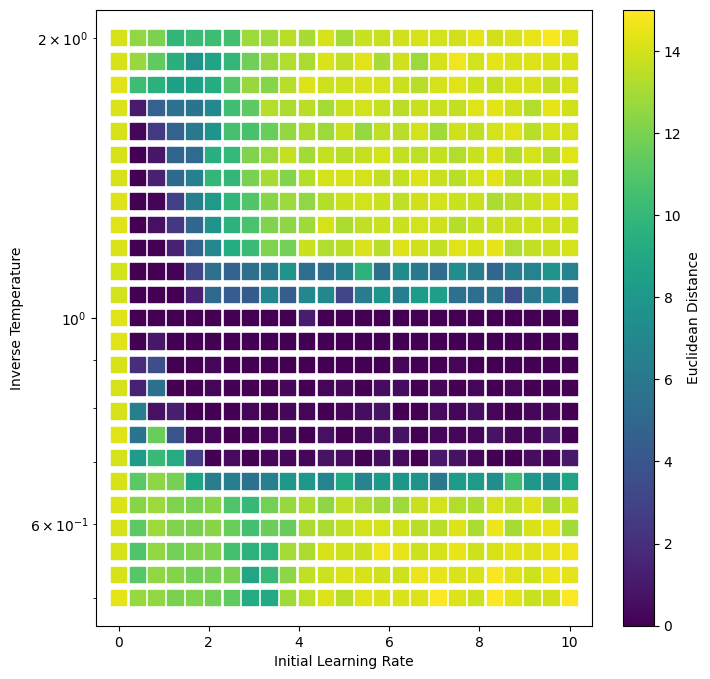}
        \caption{Fine search space}
    \end{subfigure}
    \caption{Hyperparameter search space for \(n=15\)}
\end{figure}

\begin{figure}[H]
    \centering
    \begin{subfigure}[t]{0.48\textwidth}
        \includegraphics[width=\textwidth]{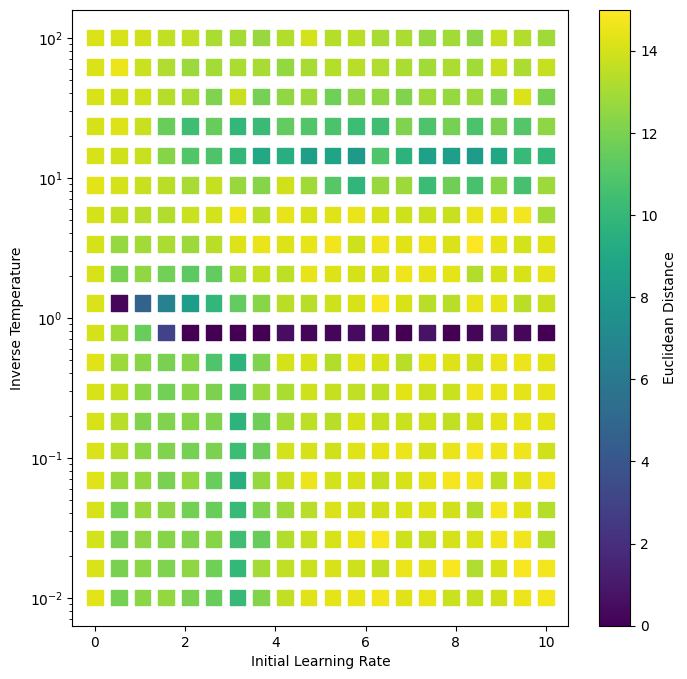}
        \caption{Coarse search space}
    \end{subfigure}
    \begin{subfigure}[t]{0.48\textwidth}
        \includegraphics[width=\textwidth]{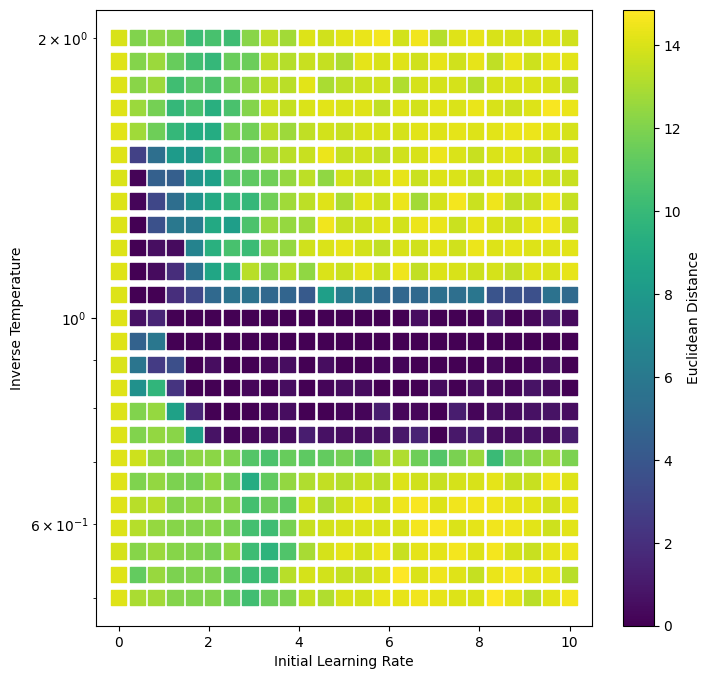}
        \caption{Fine search space}
    \end{subfigure}
    \caption{Hyperparameter search space for \(n=20\)}
\end{figure}

\begin{figure}[H]
    \centering
    \begin{subfigure}[t]{0.48\textwidth}
        \includegraphics[width=\textwidth]{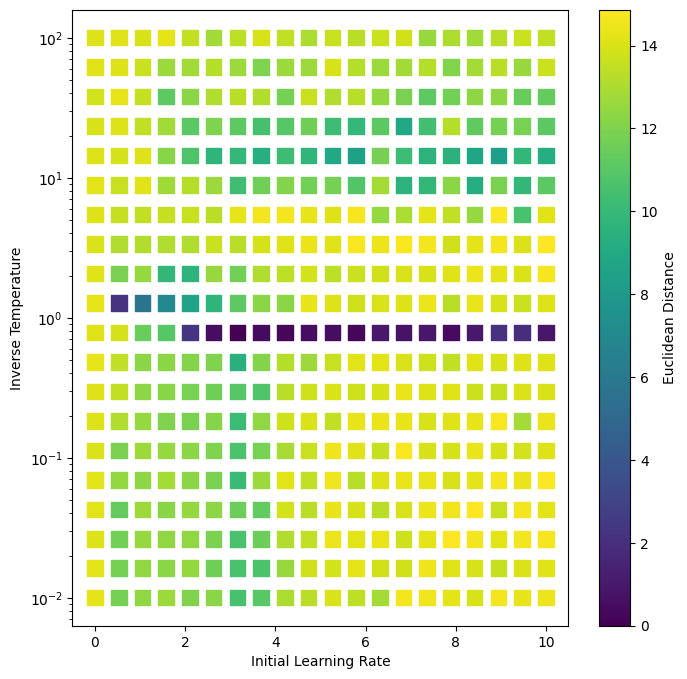}
        \caption{Coarse search space}
    \end{subfigure}
    \begin{subfigure}[t]{0.48\textwidth}
        \includegraphics[width=\textwidth]{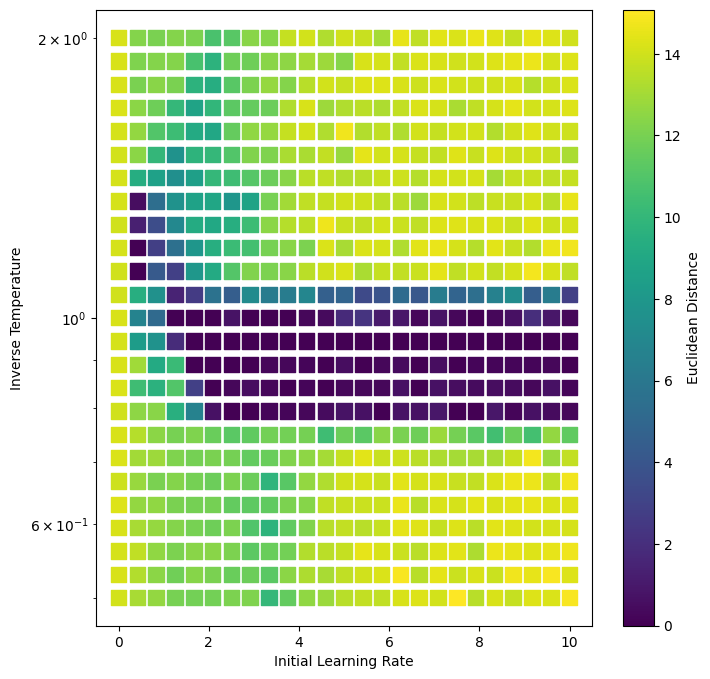}
        \caption{Fine search space}
    \end{subfigure}
    \caption{Hyperparameter search space for \(n=25\)}
\end{figure}

\begin{figure}[H]
    \centering
    \begin{subfigure}[t]{0.48\textwidth}
        \includegraphics[width=\textwidth]{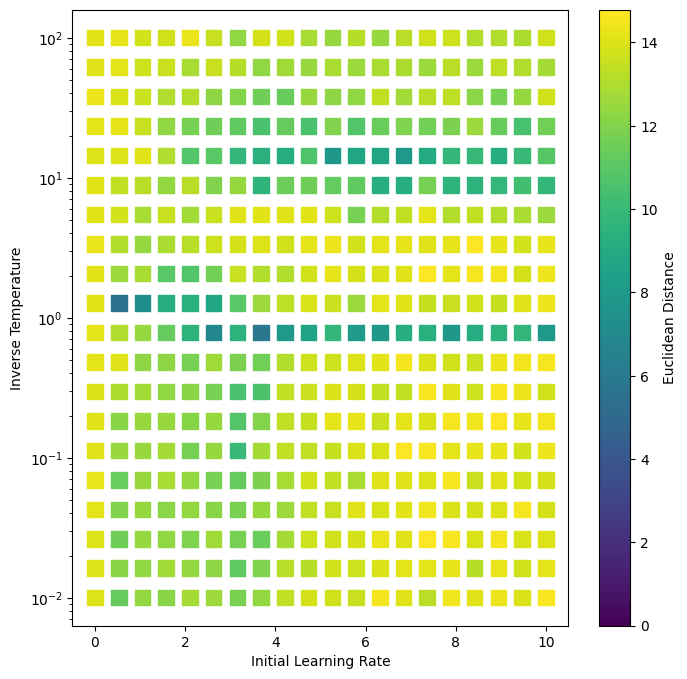}
        \caption{Coarse search space}
    \end{subfigure}
    \begin{subfigure}[t]{0.48\textwidth}
        \includegraphics[width=\textwidth]{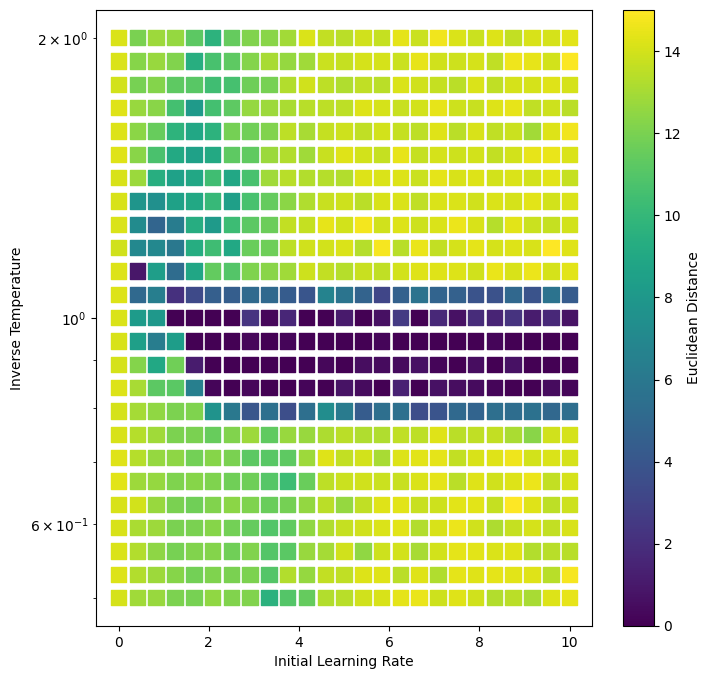}
        \caption{Fine search space}
    \end{subfigure}
    \caption{Hyperparameter search space for \(n=30\)}
\end{figure}

\subsection{Modified Network, Dimension 100, Large \(n\)}
\label{Appendix:ModifiedDim100LargeN}

These results continue with the same network and setup from Appendix \ref{Appendix:ModifiedDim100} but with much larger interaction vertices than were possible with the original network. We also present only the tight grid search results, as the coarse grid search did not capture the optimal region well.

\begin{figure}[H]
    \centering
    \includegraphics[width=0.48\textwidth]{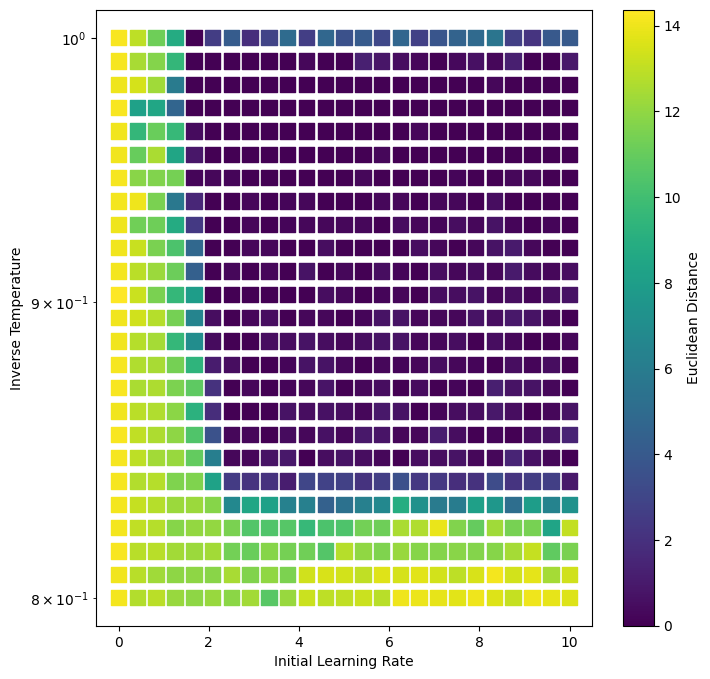}
    \caption{Hyperparameter search space for \(n=40\)}
\end{figure}

\begin{figure}[H]
    \centering
    \includegraphics[width=0.48\textwidth]{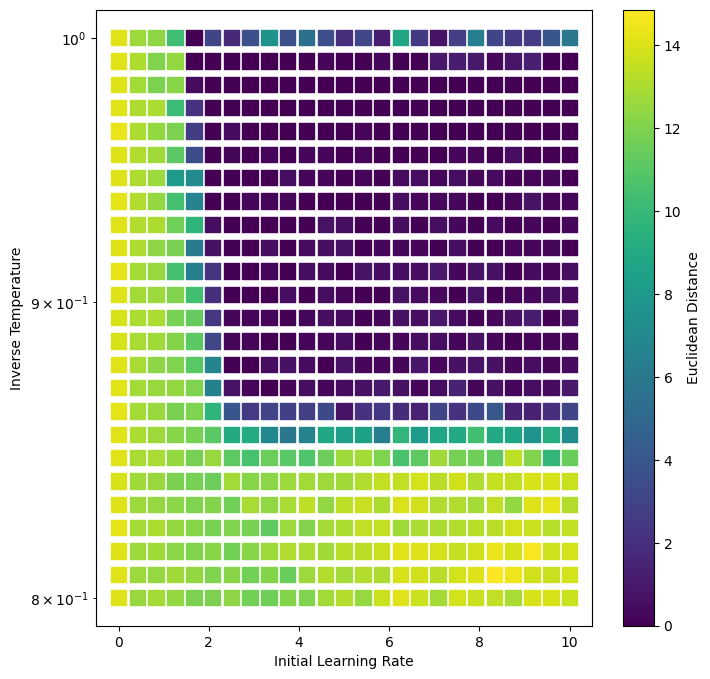}
    \caption{Hyperparameter search space for \(n=50\)}
\end{figure}

\begin{figure}[H]
    \centering
    \includegraphics[width=0.48\textwidth]{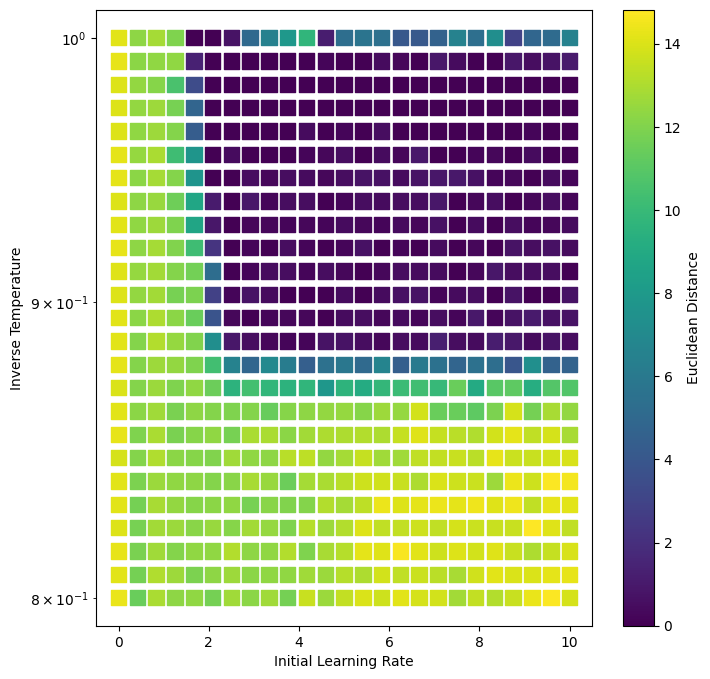}
    \caption{Hyperparameter search space for \(n=60\)}
\end{figure}

\begin{figure}[H]
    \centering
    \includegraphics[width=0.48\textwidth]{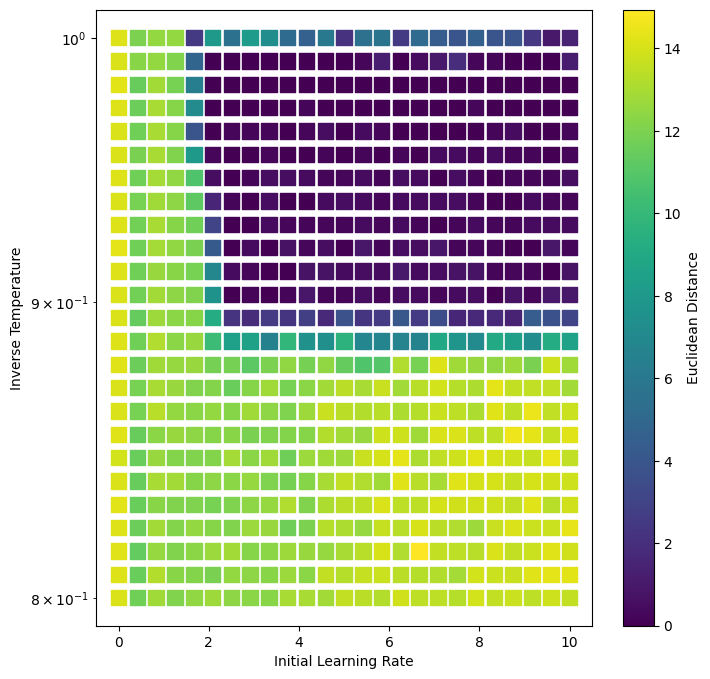}
    \caption{Hyperparameter search space for \(n=70\)}
\end{figure}

\begin{figure}[H]
    \centering
    \includegraphics[width=0.48\textwidth]{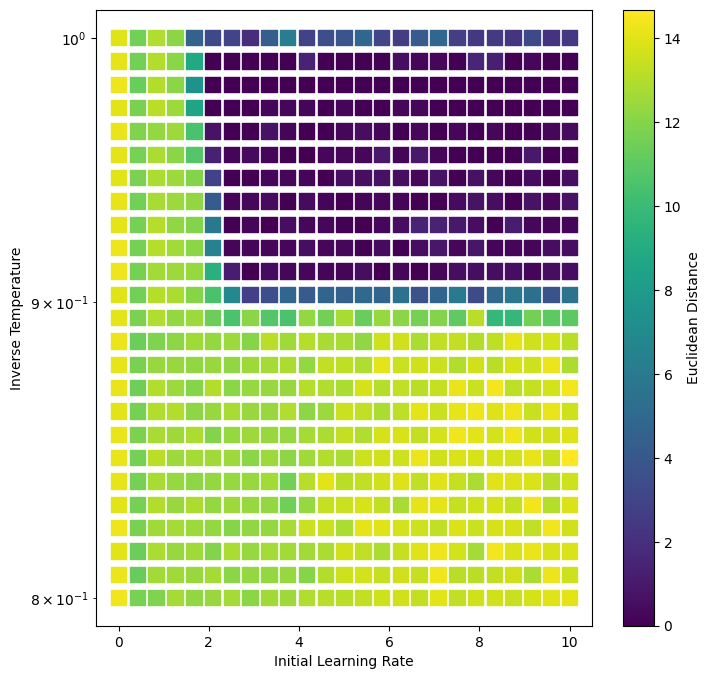}
    \caption{Hyperparameter search space for \(n=80\)}
\end{figure}

\begin{figure}[H]
    \centering
    \includegraphics[width=0.48\textwidth]{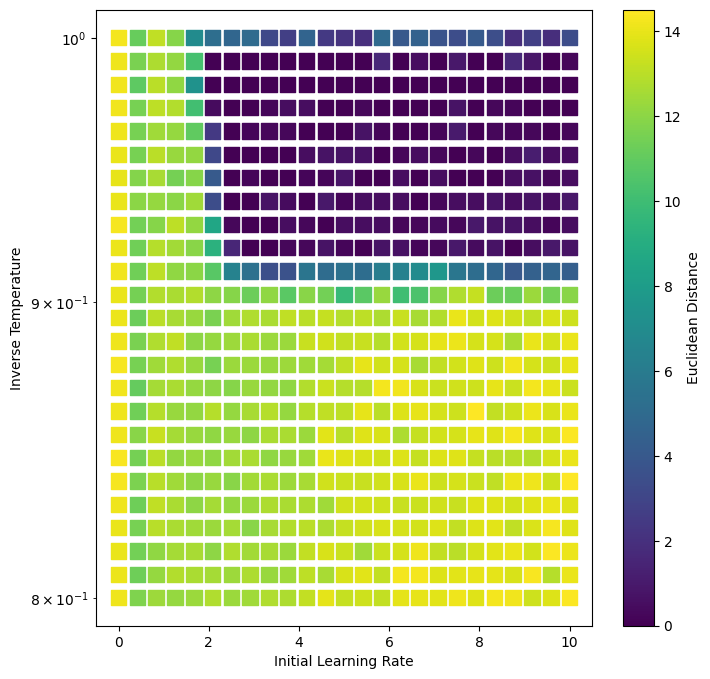}
    \caption{Hyperparameter search space for \(n=90\)}
\end{figure}

\begin{figure}[H]
    \centering
    \includegraphics[width=0.48\textwidth]{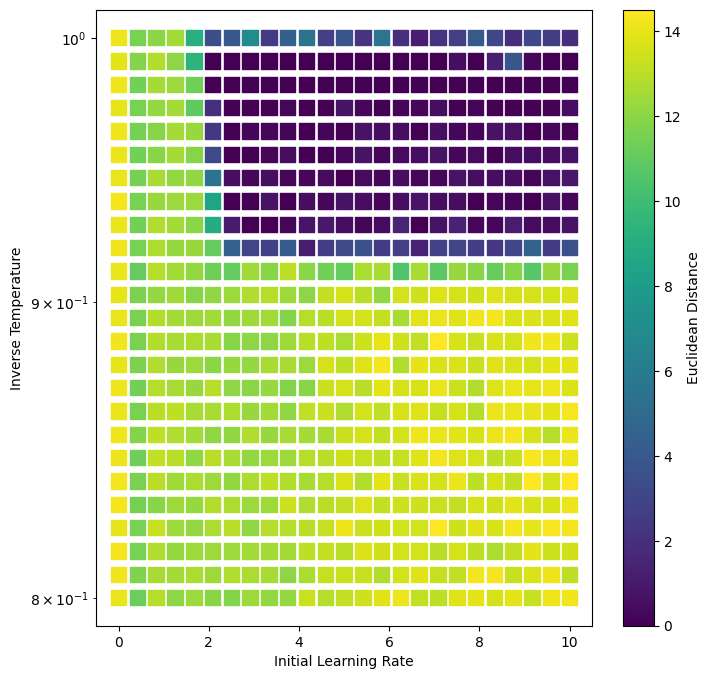}
    \caption{Hyperparameter search space for \(n=100\)}
\end{figure}

\subsection{Modified Network, Dimension 250}

These results are from our modified network, have dimension 250, and train on 30 learned states.

\begin{figure}[H]
    \centering
    \begin{subfigure}[t]{0.48\textwidth}
        \includegraphics[width=\textwidth]{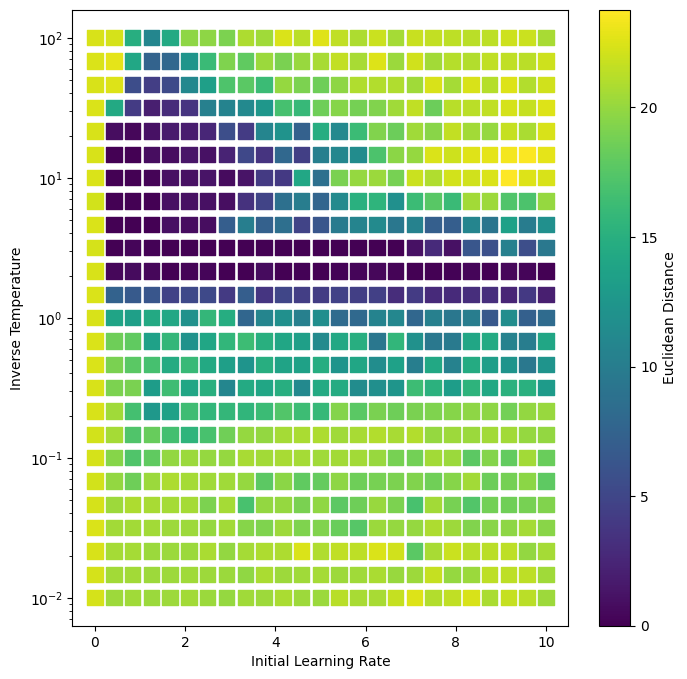}
        \caption{Coarse search space}
    \end{subfigure}
    \begin{subfigure}[t]{0.48\textwidth}
        \includegraphics[width=\textwidth]{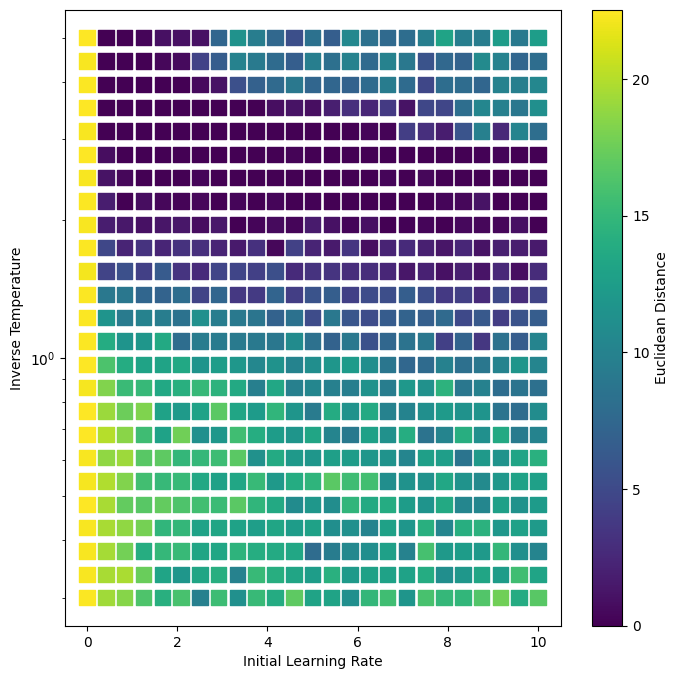}
        \caption{Fine search space}
    \end{subfigure}
    \caption{Hyperparameter search space for \(n=2\)}
\end{figure}

\begin{figure}[H]
    \centering
    \begin{subfigure}[t]{0.48\textwidth}
        \includegraphics[width=\textwidth]{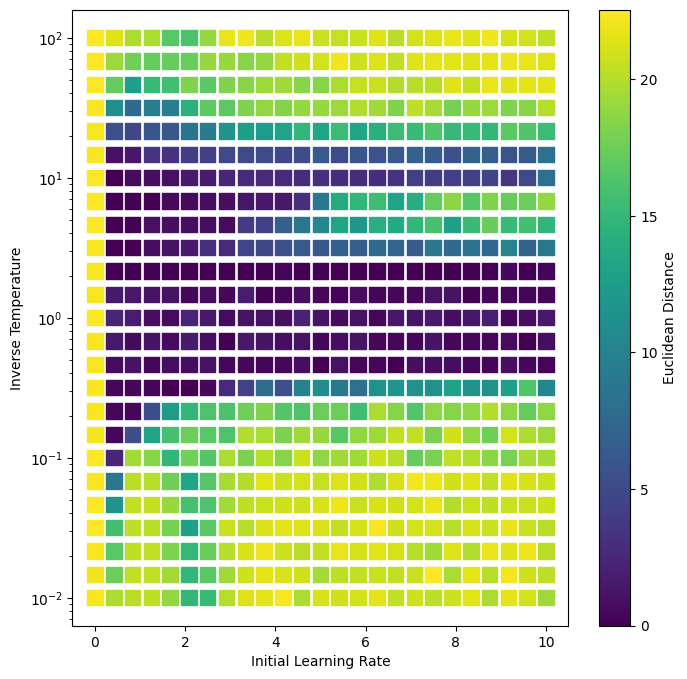}
        \caption{Coarse search space}
    \end{subfigure}
    \begin{subfigure}[t]{0.48\textwidth}
        \includegraphics[width=\textwidth]{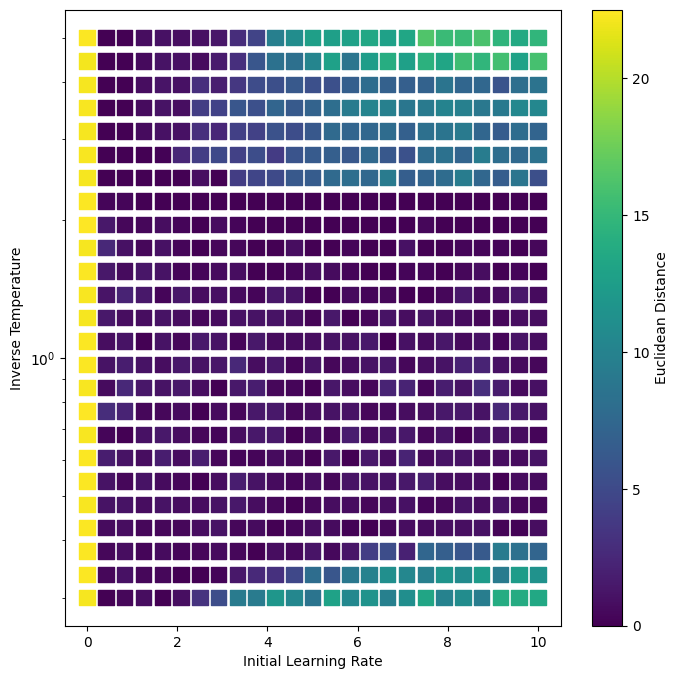}
        \caption{Fine search space}
    \end{subfigure}
    \caption{Hyperparameter search space for \(n=3\)}
\end{figure}

\begin{figure}[H]
    \centering
    \begin{subfigure}[t]{0.48\textwidth}
        \includegraphics[width=\textwidth]{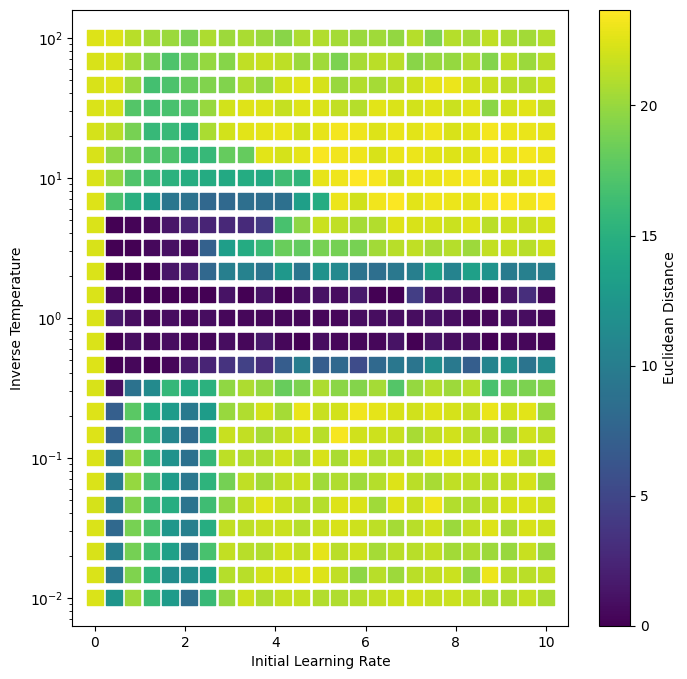}
        \caption{Coarse search space}
    \end{subfigure}
    \begin{subfigure}[t]{0.48\textwidth}
        \includegraphics[width=\textwidth]{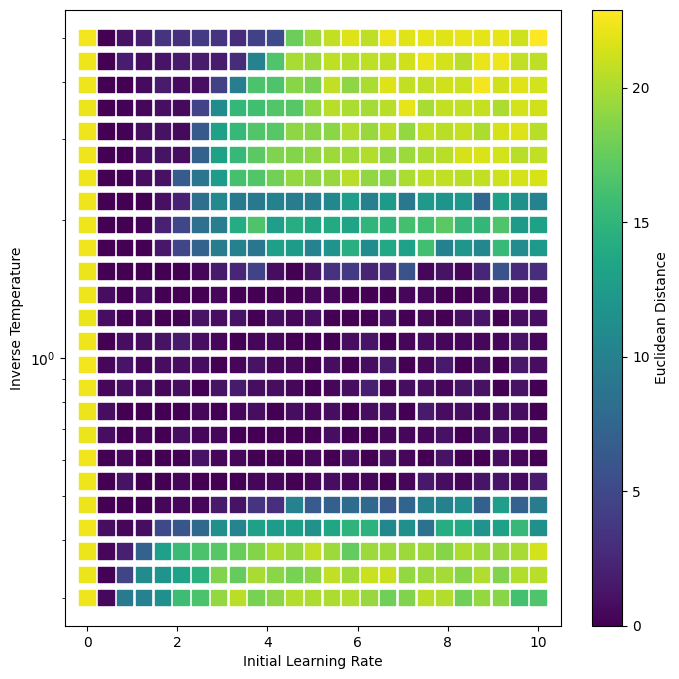}
        \caption{Fine search space}
    \end{subfigure}
    \caption{Hyperparameter search space for \(n=5\)}
\end{figure}

\begin{figure}[H]
    \centering
    \begin{subfigure}[t]{0.48\textwidth}
        \includegraphics[width=\textwidth]{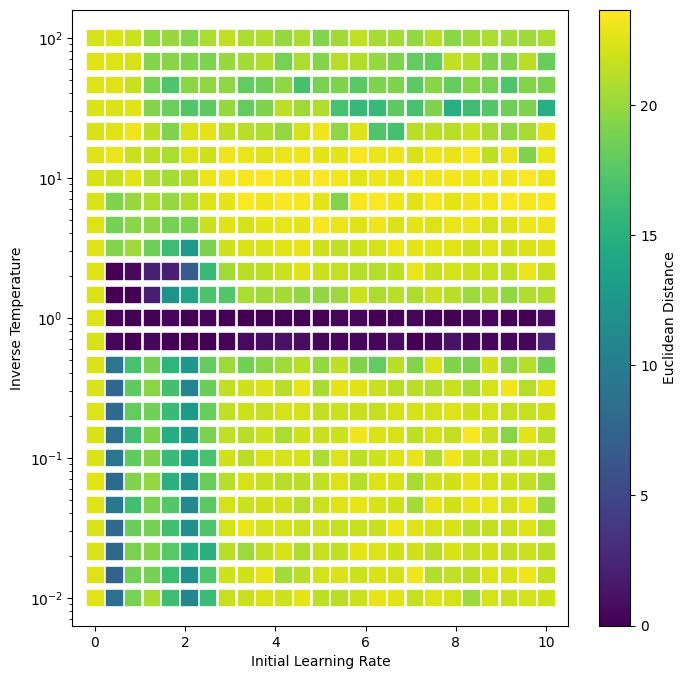}
        \caption{Coarse search space}
    \end{subfigure}
    \begin{subfigure}[t]{0.48\textwidth}
        \includegraphics[width=\textwidth]{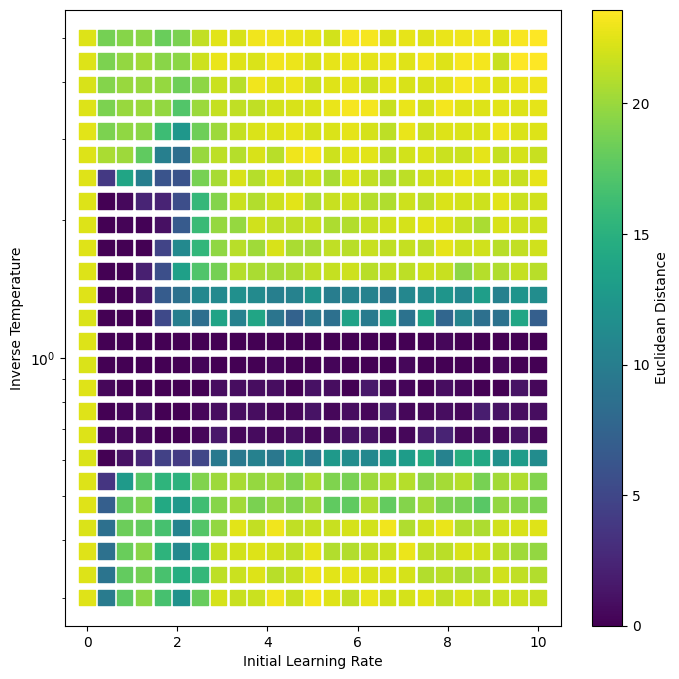}
        \caption{Fine search space}
    \end{subfigure}
    \caption{Hyperparameter search space for \(n=10\)}
\end{figure}

\begin{figure}[H]
    \centering
    \begin{subfigure}[t]{0.48\textwidth}
        \includegraphics[width=\textwidth]{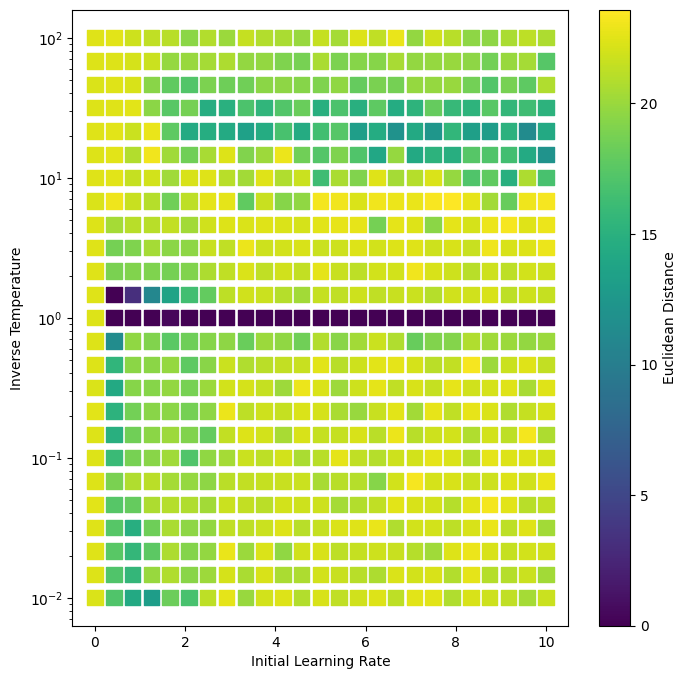}
        \caption{Coarse search space}
    \end{subfigure}
    \begin{subfigure}[t]{0.48\textwidth}
        \includegraphics[width=\textwidth]{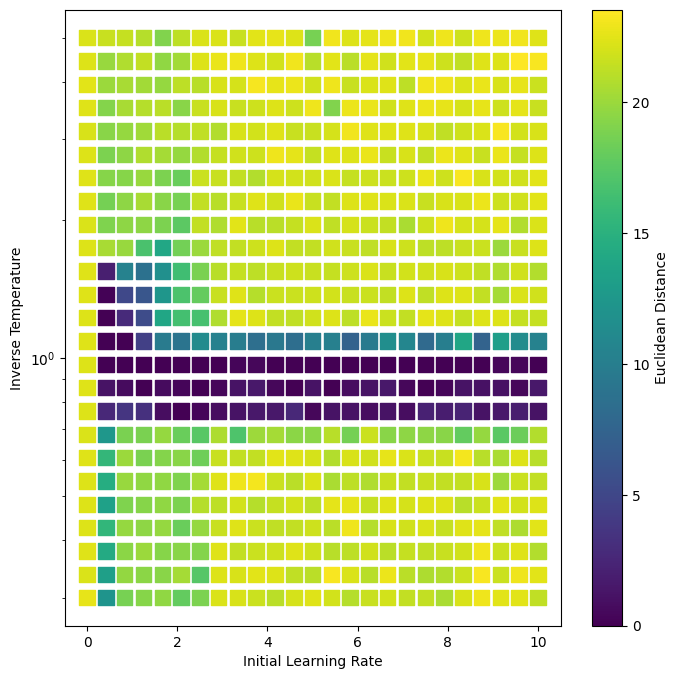}
        \caption{Fine search space}
    \end{subfigure}
    \caption{Hyperparameter search space for \(n=20\)}
\end{figure}

\begin{figure}[H]
    \centering
    \begin{subfigure}[t]{0.48\textwidth}
        \includegraphics[width=\textwidth]{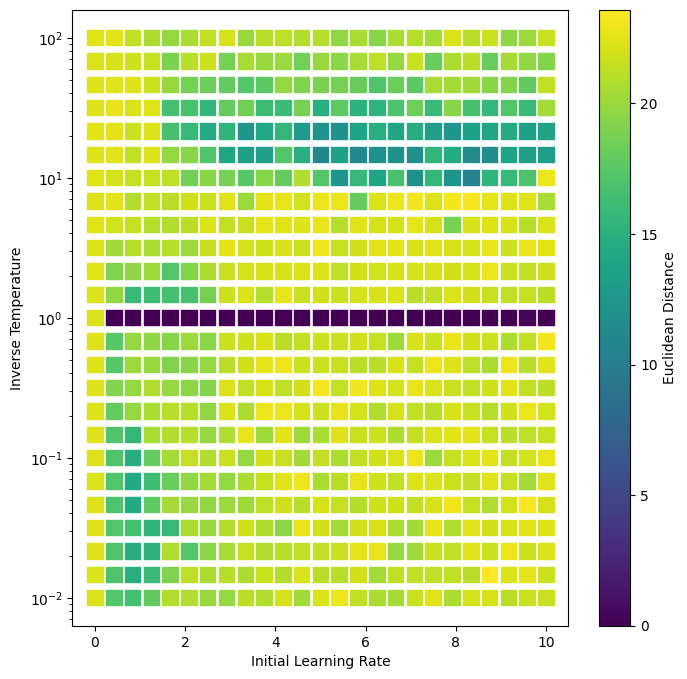}
        \caption{Coarse search space}
    \end{subfigure}
    \begin{subfigure}[t]{0.48\textwidth}
        \includegraphics[width=\textwidth]{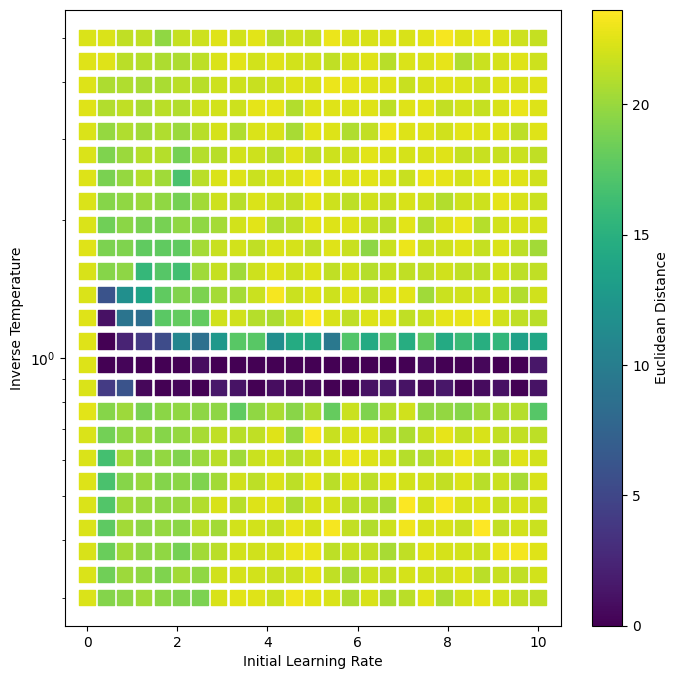}
        \caption{Fine search space}
    \end{subfigure}
    \caption{Hyperparameter search space for \(n=30\)}
\end{figure}

\end{document}